\documentclass{article}
\usepackage{algorithm}
\usepackage{algorithmicx}
\usepackage{algpseudocode}
\usepackage{amsmath}
\usepackage{amssymb}
\usepackage{amsthm}
\usepackage{bbm}
\usepackage{bm}
\usepackage{caption}
\usepackage{color}
\usepackage{dirtytalk}
\usepackage{dsfont}
\usepackage{enumerate}
\usepackage{fullpage}
\usepackage{graphicx}
\usepackage{listings}
\usepackage{mathtools}
\usepackage{natbib}
\usepackage{subfigure}
\usepackage{url}
\usepackage{xspace}
\usepackage{float}

\usepackage{authblk}

\usepackage[usenames,dvipsnames]{xcolor}
\usepackage[bookmarks=false]{hyperref}
\hypersetup{
  pdftex,
  pdffitwindow=true,
  pdfstartview={FitH},
  pdfnewwindow=true,
  colorlinks,
  linktocpage=true,
  linkcolor=Green,
  urlcolor=Green,
  citecolor=Green
}
\usepackage[capitalize,noabbrev]{cleveref}

\usepackage[textsize=tiny]{todonotes}
\usepackage{thmtools}
\declaretheorem[name=Theorem,refname={Theorem,Theorems},Refname={Theorem,Theorems}]{theorem}
\declaretheorem[name=Lemma,refname={Lemma,Lemmas},Refname={Lemma,Lemmas},sibling=theorem]{lemma}
\declaretheorem[name=Corollary,refname={Corollary,Corollaries},Refname={Corollary,Corollaries},sibling=theorem]{corollary}

\newcommand{\cE}{\mathcal{E}}

\newcommand{\cI}{\mathcal{I}}

\newcommand{\cN}{\mathcal{N}}

\newcommand{\E}[1]{\mathbb{E} \left[#1\right]}

\newcommand{\prob}[1]{\mathbb{P} \left(#1\right)}

\newcommand{\bmu}{\bm{\mu}}
\newcommand{\s}{\sigma}
\newcommand{\bs}{\bar\sigma}
\newcommand{\nr}[1]{n_{r_{#1}}}
\newcommand{\sr}{\Sigma_r}

\DeclareMathOperator*{\argmax}{arg\,max\,}

\mathchardef\mhyphen="2D

\newcommand\Ex[2]{\mathop{{\mathbb{E}_{#1}}\left[#2\right]}}
\newcommand\Prob[2]{\mathop{{\mathbb{P}_{#1}}\left(#2\right)}}
\newcommand{\bayeselim}{\ensuremath{\tt BayesElim}\xspace}
\newcommand{\bayeselimtwo}{\ensuremath{\tt BayesElim2}\xspace}
\newcommand{\ts}{\ensuremath{\tt TS}\xspace}
\newcommand{\tstwo}{\ensuremath{\tt TS2}\xspace}
\newcommand{\ttts}{\ensuremath{\tt TTTS}\xspace}
\newcommand{\freqelim}{\ensuremath{\tt FreqElim}\xspace}
\newcommand{\freqelimtwo}{\ensuremath{\tt FreqElim2}\xspace}

\title{Bayesian Fixed-Budget Best-Arm Identification}
\author[1]{Alexia Atsidakou}
\author[2]{Sumeet Katariya}
\author[1,2]{Sujay Sanghavi}
\author[3]{Branislav Kveton}
\affil[1]{University of Texas at Austin}
\affil[2]{Amazon}
\affil[3]{AWS AI Labs}

\date{}                     %% if you don't need date to appear
\date{}

\begin{document}

\maketitle

\begin{abstract}
Fixed-budget best-arm identification (BAI) is a bandit problem where the agent maximizes the probability of identifying the optimal arm within a fixed budget of observations. In this work, we study this problem in the Bayesian setting. We propose a Bayesian elimination algorithm and derive an upper bound on its probability of misidentifying the optimal arm. The bound reflects the quality of the prior and is the first distribution-dependent bound in this setting. We prove it using a frequentist-like argument, where we carry the prior through, and then integrate out the bandit instance at the end. We also provide a lower bound on the probability of misidentification in a $2$-armed Bayesian bandit and show that our upper bound (almost) matches it for any budget. Our experiments show that Bayesian elimination is superior to frequentist methods and competitive with the state-of-the-art Bayesian algorithms that have no guarantees in our setting.
\end{abstract}

\section{Introduction}
\label{sec:introduction}

\emph{Best-arm identification (BAI)} is a bandit problem where the goal is to identify the optimal arm after a number of observations. It has many applications in practice, such as online advertising, recommender systems, and vaccine tests \citep{lattimore-Bandit}. In the \emph{fixed-budget (FB)} setting \citep{audibert-2010-BAI}, the goal is to identify the optimal arm within a fixed budget of $n$ observations. This setting is common in applications where the observations are costly, such as in Bayesian optimization \citep{krause08nearoptimal}. In the \emph{fixed-confidence (FC)} setting \citep{ActionElimination-Evendar2006a,soare2014bestarm}, the goal is to find the optimal arm with a given confidence level, while minimizing the sample complexity. Some works even studied both settings \citep{Gabillon-2012,Karnin2013AlmostOE,KaufmannEmilie2016OnTC}.

Most BAI algorithms, including all of the aforementioned, have been proposed for the frequentist setting. There, the bandit instance is chosen potentially adversarially from some hypothesis class (such as Bernoulli rewards or linear models) and the agent tries to identify the best arm by using only the knowledge of the class. While frequentist BAI algorithms have strong guarantees in their setting, they do not exploit the prior information in the Bayesian setting. As a result, their Bayesian guarantees do not improve as the prior becomes more informative. While Bayesian algorithms can naturally do that \citep{pmlr-v33-hoffman14,russo2016simple}, to the best of our knowledge the Bayesian guarantees of all state-of-the-art Bayesian BAI algorithms are distribution-independent. Therefore, while in practice the algorithms can benefit from the prior information, their theoretical analyses do not show the improvement.

In this work, we set out to address the obvious gap in prior works on Bayesian BAI. Specifically, we propose the first Bayesian BAI algorithm for the fixed-budget setting that uses a prior distribution over bandit instances as a side information, and also has an error bound that improves with a more informative prior. This work parallels modern analyses of Thompson sampling in the cumulative regret setting \citep{russo14learning,russo16information,hong22thompson,hong22hierarchical}. For instance, \citet{russo14learning} showed that the $n$-round Bayes regret of linear Thompson sampling is $O(\sqrt{d})$ lower than the best known regret bound in the frequentist setting \citep{agrawal13thompson}. \citet{hong22hierarchical} showed that the hierarchy of model parameters in meta- and multi-task bandits reduces the $n$-round Bayes regret of Thompson sampling. We believe that our work lays foundations for similar future improvements in Bayesian BAI.

This paper makes the following contributions. First, we formulate the setting of Bayesian fixed-budget BAI with $K$ arms and propose an elimination algorithm for it. The algorithm is a variant of successive elimination \citep{Karnin2013AlmostOE} where the \emph{maximum likelihood estimate (MLE)} of the mean arm reward is replaced with a Bayesian \emph{maximum a posteriori (MAP) estimate}. We call the algorithm \bayeselim. Second, we prove an upper bound on the probability that \bayeselim fails to identify the optimal arm. Our proof technique is novel and allows us to obtain a probability bound, which is distribution-dependent and reflects the quality of the prior. To achieve this, we use a frequentist-like analysis, where prior quantities appear as bias terms, and then integrate out the random instance at the end. This technique differs from Bayesian analyses in the cumulative regret setting \citep{russo14learning,russo16information,hong22thompson,hong22hierarchical}, which only lead to finite-time guarantees that are far from the lower bounds \citet{lai87}. Third, we prove a distribution-dependent lower bound for $K = 2$ arms. Our upper and lower bounds match for any fixed budget. Note that even in frequentist BAI, the case of $K=2$ arms with Gaussian rewards is the only one with matching upper and lower bounds \citep{KaufmannEmilie2016OnTC,Kato22}. The existence of matching bounds beyond this setting is an important open question \citep{qin22a}. Finally, we analyze the frequentist counterpart of \bayeselim and show that its Bayesian guarantees are always worse than those of \bayeselim. Our synthetic and real-data experiments show the superiority of Bayesian over frequentist BAI algorithms and empirically confirm the advantage of using the prior.

One surprising property of our upper and lower bounds is that they are proportional to $1 / \sqrt{n}$, where $n$ is the budget. At a first sight, this seems to contradict to the frequentist upper \citep{Karnin2013AlmostOE} and lower \citep{Carpentier16} bounds, which are proportional to $\exp[- n\Delta^2]$, where $\Delta$ is the gap. The reason is that the frequentist bounds are proved in a harder setting, per instance instead of on average over instances, yet they are lower. However, our bounds are compatible with these results. Roughly speaking, even the frequentist bounds integrated over $\Delta \sim \mathcal{N}(0, 1)$ yield $1 / \sqrt{n}$, since the budget $n$ in $\exp[-n \Delta^2]$ plays the role of reciprocal variance in the Gaussian integral. We state this more rigorously later.

One line of work closely related to our setting is the \textit{ranking and selection} problem. There, optimal budget-allocation ratios are defined using unknown reward parameters of the arms and the objective is to achieve these unknown rates asymptotically. In the setting of Gaussian rewards and priors, \citet{ryzhov16} studied a family of policies called expected improvements methods and \citet{chen2016EI} showed that a variant of these methods achieves optimal allocation ratios asymptotically. Finally, in terms of asymptotic guarantees, a lot of progress was made recently in a related \textit{Bayesian pure exploration} setting, where the goal is to minimize the expected gap after a budget of $n$ observations \citep{bubeck2010pure}, namely the \textit{simple regret}. 
Specifically, in the case of Bernoulli rewards, \citet{komiyama2021optimal} derived a Bayesian lower bound on the simple regret and proposed a frequentist algorithm that matches it asymptotically. The relations between this objective to ours, as well as with \textit{cumulative regret} are discussed in \cref{sec:pure exploration}. In our work, rather than studying the asymptotic behavior, we focus on finite-budget guarantees. Our proposed algorithm has error guarantees that (almost) match the lower bound for \textit{any fixed budget}.

\section{Problem Setting}

Bayesian fixed-budget best-arm identification (BAI) involves $K$ arms with \textit{unknown} mean rewards $\bmu=(\mu_1,...,\mu_K)$ drawn from a \textit{known} prior distribution $H$.
After arm $i\in [K]$ is pulled, a policy observes a sample $X_i$ from its reward distribution.
We focus on the Gaussian case, where the reward of each arm $i\in [K]$ follows a Gaussian distribution with known variance, i.e. $X_{i}\sim \mathcal{N}(\mu_i,\s_i^2)$, and its mean is drawn independently from a known prior $\mu_i\sim \mathcal{N}(\nu_{i}, \s_{0}^2)$. We refer to $\s_i^2$ as the reward variance of arm $i$; and to $\nu_i$ and $\s_0^2$ as the prior mean and variance of the {mean} reward of arm $i$, respectively. A policy interacts with the arms for $n$ exploration rounds, where $n$ is a known budget, with the goal of identifying the optimal arm $i_*(\bmu)=\argmax_{i\in[K]}\mu_i$. We denote by $J$ the arm recommended by the policy at the end of $n$ rounds. For any fixed mean reward vector $\bmu$, the \textit{probability of misidentification} is
${\prob{J\neq i_*(\bmu)|\bmu}}$, where $\prob{\cdot|\bmu}$ is over the randomness of the policy and the reward realizations of each round, i.e. considering the mean rewards fixed. The setting where the mean reward vector $\bmu$ is fixed corresponds to the frequentist BAI setting. There the objective of a policy is to minimize the worst-case probability of misidentification for any possible fixed vector of mean rewards.
In contrast to the frequentist setting, the performance of a policy in Bayesian BAI is measured by its \textit{expected} {probability of misidentification}, i.e.
\begin{align}\label{eq:objective}
    \Ex{H}{\prob{J\neq i_*(\bmu)|\bmu}}
\end{align}
where the expectation is taken w.r.t. the prior distribution of the mean rewards.

\paragraph{Notation.} We denote by $H(\bmu)$ and $h(\bmu)$ the cumulative probability and density of mean vector $\bmu$, respectively. We denote the optimal arm by $i_*=i_*(\bmu)$, and its mean reward and prior mean are $\mu_* = \mu_{i_*}$ and $\nu_* = \nu_{i_*}$, respectively. We note that $i_*$, $\mu_*$, and $\nu_*$ are functions of $\bmu$ and thus random variables. We denote by $\bar \mu_{i,m}$ and $\bs_{i,m}^2$ the posterior mean and variance, respectively, of arm $i\in[K]$ given $m$ i.i.d. observations from its reward distribution. We summarize our notation in \cref{table:1} in \cref{app:notation_tab}.

\section{Bayesian Elimination}

We propose \bayeselim, a Bayesian successive elimination algorithm inspired by the frequentist algorithm of \citet{Karnin2013AlmostOE}. \bayeselim incorporates prior information and eliminates arms based on the maximum a posteriori estimates of their mean rewards. For any arm $i\in [K]$, the posterior distribution of its mean reward given $m$ i.i.d.\ observations, $X_{i,s}\sim \cN(\mu_i,\s_i^2)$ for $s\in[m]$, is a Gaussian distribution $\mathcal{N}\left( \bar \mu_{i,m}, \bs_{i,m}^2 \right)$ with mean
\begin{align}
    \bar \mu_{i,m}=\bs_{i,m}^2 \left(\frac{\nu_i}{\s_{0}^2}+\frac{\sum_{s\in [m]} X_{i,s}}{\s_i^2}\right)
    \label{eq:mubarim}
\end{align}
and variance $\displaystyle \bs_{i, m}^2 = (1 / \s_{0}^2 + m / \s_i^2)^{-1}$ \citep{Murphy07}.

\bayeselim splits the exploration budget $n$ into $R=\lceil\log_2(K)\rceil$ elimination rounds of equal budget per round, $\lfloor\frac{n}{R}\rfloor$. At each elimination round $r\in[R]$, the algorithm maintains a set of active arms, denoted by $S_r$. 
At the end of each round, after collecting samples from the arms, the algorithm eliminates a half of the active arms that have the lowest posterior means. 
The arms that survive elimination at round $r$ remain active in the next round.
In each round $r$, \bayeselim allocates the per-round budget among the arms in $S_r$ according to their reward variances. In particular, each arm $i\in S_r$ is sampled $\lfloor n_{r,i} \rfloor$ times, where
\begin{align}
    n_{r,i}=\frac{n}{R}\frac{\s_i^2}{\sum_{k\in S_r}\s_k^2}.
    \label{eq:split}
\end{align} 
\bayeselim is presented in \cref{alg:bayesian_successive_elimination}. 
By \cref{eq:split}, a larger sample size is allocated to arms with a higher reward noise. In particular, when the per-round budget is split according to \cref{eq:split}, the posterior variances of all active arms at the end of each round are equal, that is
\begin{align*}
    \bs_{i,n_{r,i}}^2=\left(\frac{1}{\s_{0}^2}+\frac{n_{r,i}}{\s_i^2}\right)^{-1} = \left(\frac{1}{\s_{0}^2}+\frac{n}{R\sum_{k\in S_r}\s_k^2}\right)^{-1}, \, \forall i\in S_r.
\end{align*} 
% for all $i\in S_r$.
Finally, in the special case of equal reward variances, i.e. $\s_i=\s$ for all $i\in[K]$ and some $\sigma \geq 0$, the budget is distributed uniformly among the active arms in each round. This is the same allocation as in the original algorithm of \citet{Karnin2013AlmostOE}.
For more details on the allocation rule in \cref{eq:split} and a comparison with the optimal allocation rule in \citet{KaufmannEmilie2016OnTC} for frequentist $2$-armed Gaussian BAI, see \cref{app:allocation}.

\begin{algorithm}[t]
\caption{\bayeselim: Bayesian elimination.}\label{alg:bayesian_successive_elimination}
\begin{algorithmic}[1]
\State Let $R \leftarrow \lceil\log_2K\rceil$
\State Initialize $S_1 \leftarrow [K]$ 
\For{$r=1,..., R$}
\For{$i\in S_r$}
\State Get $\lfloor n_{r,i}\rfloor$ samples of arm $i$, where $n_{r,i}$ is given by \cref{eq:split}
\State Compute posterior mean $\bar\mu_{i,n_{r,i}}$ using \cref{eq:mubarim}
\EndFor
\State Set $S_{r+1}$ to be the set of $\lceil|S_r|/2\rceil$ arms in $S_r$ with largest posterior means $\{\bar\mu_{i,n_{r,i}}\}_{i\in S_r}$
\EndFor
\end{algorithmic}
\end{algorithm}

\section{Analysis}\label{sec:ub_analysis}
We provide a distribution-dependent upper bound for the expected probability of misidentification of \bayeselim. 
Our approach is novel and deviates from
Bayesian bandit analyses for cumulative regret \citep{russo14learning,russo16information,hong22thompson,hong22hierarchical}, which condition on history and bound the regret in expectation over the posterior in each round.
In our method, the mean vector $\bmu$ is initially considered fixed, allowing the use of frequentist-like concentration arguments to bound the error due to sampling. However, in contrast to frequentist bounds, our bounds involve gap expressions which contain prior quantities as bias terms. These biased-gap expressions are carefully selected such that the relaxation is negligible in expectation over priors, while the result remains integrable.
After performing this analysis, we integrate out the randomness due to prior information. 
% It also permits a direct comparison with the guarantees of frequentist techniques applied to the Bayesian setting. 

We now state our main upper bound on the expected probability of misidentification for our algorithm. All omitted proofs of this section are in \cref{app:UB}. 
\begin{theorem}\label{thm:bound_with_integral}
The expected probability of misidentification for \bayeselim is bounded as
\begin{align*}
    &\Ex{H}{\prob{J\neq i_*|\bmu}} \leq 2\log_2 K \sqrt{\frac{R\sum_{k\in[K]}\s_k^2}{n\sigma_0^2+ R\sum_{k\in[K]}\s_k^2}} \sum_{i\in[K]}\sum_{j>i}  \exp\left(-\frac{(\nu_i-\nu_j)^2}{4\s_0^2}\right).
\end{align*}
\end{theorem}
\noindent We elaborate on the dependence on the budget and prior quantities below.

\textbf{Dependence on $n$.} The bound in \cref{thm:bound_with_integral} is $O(1/\sqrt{n})$, which contrasts with $e^{-O(n)}$ in frequentist BAI \citep{ActionElimination-Evendar2006a}. This may at first seem counterintuitive due to the additional access to prior information available to the learner. However, this dependence is inherent to the objective of expected probability of misidentification defined in \cref{eq:objective}, which integrates the probability of misidentification over the possible instances. This is in contrast to the frequentist setting, where the objective is the worst-case probability of misidentification on a single instance. 

Also note that integrating an optimal frequentist bound over all instances would result in an $O(1/\sqrt{n})$ dependence. This can be illustrated in the following simple example ($\s_0=1$),
\begin{align*}
    \Ex{H}{\exp\left({-n(\mu_i-\mu_j)^2}\right)}&=\frac{1}{2\pi}\int_{\mu_i, \mu_j} \exp\left({-n(\mu_i-\mu_j)^2-\frac{(\nu_i-\mu_i)^2}{2}-\frac{(\nu_j-\mu_j)^2}{2}}\right)\, d\mu_j \, d\mu_i\\
    &= 
    \frac{1}{\sqrt{4n+1}}
    \exp\left(-\frac{n}{4n+1}(\nu_i-\nu_j)^2\right),
\end{align*} 
where the last step results from completing the square in the exponent and computing the Gaussian integral. This dependence of \bayeselim on budget $n$ is optimal (\cref{sec:lb'}).

\textbf{Dependence on prior quantities.} The bound of \cref{thm:bound_with_integral} decreases exponentially with the squared gaps between the prior means normalized by the prior variance, ${(\nu_i-\nu_j)^2}/{\s_0^2}$. This agrees with the intuition that as the prior mean gaps grow, or the prior variances decrease, the uncertainty in the parameters decreases and the problem of identifying the best arm becomes easier. 

\textbf{Case $\s_0\rightarrow 0$:} As the prior variance $\s_0^2$ decreases, the Gaussian priors on mean arm rewards become more concentrated and separated. The randomness in the mean arm rewards diminishes and the Bayesian algorithm is less likely to recommend a suboptimal arm. In particular, in the extreme case of $\s_0 = 0$ the learner has exact knowledge of the mean without any sampling. Then, the expected probability of an error is $0$. In this case, as expected, the upper bound on the probability of misidentification of \bayeselim becomes $0$. This is the first work in Bayesian BAI to reflect this.

\textbf{Case $\s_0\rightarrow \infty$:} A less-obvious property of the bound in \cref{thm:bound_with_integral} is that it becomes $0$ as $\s_0\rightarrow\infty$. Specifically, as $\sigma_0 \rightarrow \infty$, the gaps between the arm means become infinitely large with a high probability. Thus, as long as the reward variances $\s_i^2$ are finite, the arms can be distinguished with $O(1)$ samples and the probability of misidentification goes to zero for any $n>0$. This behavior is inherent in the Bayesian objective of \cref{eq:objective} and consistent with the lower bound. \footnote{This paper can be extended to varying prior variances $\sigma^2_{0,i}$ for $i\in[K]$. Then the same discussions would apply to $\max_i \sigma_{0,i}\rightarrow 0$ or $\min_i \sigma_{0,i}\rightarrow \infty$.}

\subsection{Frequentist Algorithm for Bayesian BAI}
\label{sec:frequentist bai}

To illustrate the significance of properly incorporating prior knowledge, we compare the guarantees for \bayeselim with those for frequentist elimination \citep{Karnin2013AlmostOE} in the Bayesian BAI setting. Since the algorithm in \citet{Karnin2013AlmostOE} is designed for equal reward variances, we extend it using the allocation rule in \cref{eq:split}. This algorithm eliminates based on sample averages at the end of each round. We remark that this is the best algorithm for frequentist BAI, in terms of worst-case guarantees for the probability of misidentification. This algorithm can be also viewed as \bayeselim where we take $\s_0 \rightarrow \infty$ in the computation of the posterior mean. 

We first derive an upper bound on the expected probability of misidentification for the frequentist algorithm following a similar analysis to that of \bayeselim. 
\begin{restatable}{theorem}{thmFreqUBound}\label{thm:freq_bound}
The (generalization of the) elimination algorithm of \citet{Karnin2013AlmostOE} satisfies
\begin{align*}
    \Ex{H}{\prob{J\neq i_*|\bmu}} 
    &\leq  2\log_2 K  \sqrt{\frac{R \sum_{k\in[K]}\s_k^2}{n{\s_0^2}+R\sum_{k\in[K]}\s_k^2}} \nonumber\\
    &\times \sum_{i\in[K]}\sum_{j>i}\exp\left(-\frac{n\s_0^2}{n\s_0^2+{R}\sum_{k\in[K]}\s_k^2}\cdot\frac{(\nu_i-\nu_j)^2}{4\s_0^2}\right).\nonumber
\end{align*}
\end{restatable}
\noindent We compare the two bounds generally, as well as when the prior becomes more informative.

\textbf{Dependence on $n$.} The above bound is also of $O({1}/{\sqrt{n}})$. The difference lies in the multiplicative terms in the exponent $\frac{n\s_0^2}{n\s_0^2+{R}\sum_{k\in[K]}\s_k^2}\leq 1$. Due to them, the guarantee of \bayeselim is \textit{always} better than that of its frequentist counterpart, which does not use the prior. The difference diminishes as the budget increases. Specifically, as $n\rightarrow\infty$, we get $\frac{n\s_0^2}{n\s_0^2+{R}\sum_{k\in[K]}\s_k^2}\rightarrow 1$, and the two bounds coincide asymptotically. This agrees with the intuition that the prior information is dominated by data in the asymptotic regime of $n \to \infty$, and thus becomes less important.

\textbf{Dependence on prior quantities.} For any fixed budget $n$, as the prior variance becomes smaller the guarantee of \bayeselim becomes \textit{arbitrarily} better than that of frequentist elimination. We note that as we take $\s_0\rightarrow 0$ in \cref{thm:bound_with_integral}, the probability of misidentification for \bayeselim tends to $0$. In contrast, the bound in \cref{thm:freq_bound} becomes a constant,
\begin{align*}
    2\log_2 K\sum_{i\in[K]}\sum_{j>i}\exp\left(-\frac{1}{4}\frac{n}{{R\sum_{k\in[K]}\s_k^2}}(\nu_i-\nu_j)^2\right).
\end{align*}
This is, intuitively, a bound on the error of the frequentist algorithm when facing a single instance $\mu_i= \nu_i$ but without using that knowledge. This comparison of \cref{thm:bound_with_integral,thm:freq_bound} shows that the benefit of using Bayesian algorithms can be huge as the prior becomes more informative.

The above discussion of taking $\s_0\rightarrow 0$ in the frequentist algorithm only applies to the upper bound on its expected probability of misidentification. In \cref{sec:lb'}, we present a formal lower bound for any frequentist policy applied to the Bayesian setting, which formally quantifies the loss of ignoring prior information. 
% We sketch the proof of \cref{thm:bound_with_integral} next.

\subsection{Proof Sketch of \cref{thm:bound_with_integral}} 
To prove \cref{thm:bound_with_integral}, we first consider a fixed parameter vector $\bmu$. 
For any fixed round $r\in [R]$, in \cref{lem:posterior_mean_concentration} we bound the probability that the posterior mean of some suboptimal arm $i$ is larger that the posterior mean of the optimal arm of $\bmu$. To achieve this, we use a relaxation that may not be tight for every single instance $\bmu$. However, as we show later, it is negligible in expectation over prior means. Then, for any fixed $r\in[R]$ and $\bmu$, in \cref{lem:r_elimination} we bound the probability that the optimal arm is eliminated at round $r$. Subsequently, we chain all inequalities and bound the probability of error in instance $\bmu$. Up to this point, our analysis is frequentist-like and treats $\bmu$ as fixed. Finally, in \cref{lem:final_ub}, we integrate out the randomness of $\bmu$.
To simplify the analysis, we ignore errors due to rounding of the budget and treat $\log_2K$ and $n_{r,i}$ as integers. 

We first introduce the following biased-gap expressions, which are used to bound to the probability of error
\begin{align}\label{eq:fn_ub}
    &g_r(i,j) =n\frac{(\mu_{i}- \mu_j)^2}{4R \sum_{k\in S_r}\s_k^2 }+\frac{(\nu_{i}-\nu_j)(\mu_{i}- \mu_j)}{2\s_0^2}.
\end{align} 
The first term involves the squared gap between the arm means, which appears in frequentist BAI analyses. The latter is a bias due to the prior. Notice that the above expression can be negative when $(\nu_i-\nu_j)(\mu_i-\mu_j)<0$.

Note that the set of active arms in any round $r\in [R]$, i.e. $S_r$, is a random variable that depends on the reward realizations as well as the randomness in the parameters $\bmu$ of the reward distributions of the arms. For any fixed parameters $\bmu$ and round $r\in[R]$, the following lemma bounds the probability that the posterior mean $\bar\mu_{i_*,n_{r,{i_*}}}$ of the best arm $i_*$ of $\bmu$ in round $r$ is smaller that the posterior mean $\bar\mu_{i,n_{r,i}}$ of some arm $i\in S_r$. This argument is frequentist-like because it only deals with the randomness of the reward realizations.
\begin{restatable}{lemma}{lemPostManConc}\label{lem:posterior_mean_concentration}
Fix instance $\bmu$ and round $r \in [R]$. Suppose that $i_* \in S_r$. Then, for any $i \in S_r$ we have
\begin{align*}
&\prob{\bar\mu_{i,n_{r,i}}>\bar\mu_{i_*,n_{r,i_*}}|\bmu}  \leq\exp\left(-g_r(i_*,i)\right).
\end{align*} 
\end{restatable}

Now we bound the probability that the optimal arm is eliminated at some round $r\in[R]$ for a fixed parameter instance $\bmu$. This bound follows from \cref{lem:posterior_mean_concentration} and the fact that if arm $i_*$ is eliminated at round $r$, then at least $\left\lfloor|S_r|/2\right\rfloor$ arms have larger posterior means than $i_*$ at the end of the round.
\begin{restatable}{lemma}{lemRElimination}\label{lem:r_elimination}
Fix instance $\bmu$ and round $r\in [R]$. Then, there exists some $j_{r,\bmu}\in S_r \setminus\{i_*\}$ such that the probability that $i_*$ is eliminated at round $r$ satisfies
\begin{align*}
    &\prob{i_* \not \in S_{r+1}| \{i_* \in S_r\},  \bmu}  \leq 2 \exp\left(-g_r(i_*,j_{r,\bmu})\right).
\end{align*}
\end{restatable}
Putting together the above lemmas, the expected probability of misidentification is bounded as
\begin{align}\label{eq:regret_bound_1}
    \Ex{H}{\prob{J\neq i_*|\bmu}} &= \int_{\bmu} \prob{J\neq i_*|\bmu} h(\bmu) \,d\bmu \nonumber\\
    &\leq \int_{\bmu} \sum_{r\in [R]}\prob{i_*(\bmu) \not \in S_{r+1}| \bmu, \{i_*(\bmu) \in S_r\}} h(\bmu) \,d\bmu \nonumber\\
    &\leq 2\sum_{r\in [R]} \int_{\bmu}
    \exp\left(-g_r(i_*,j_{r,\bmu})\right) h(\bmu) \,d\bmu \quad\text{(by \cref{lem:r_elimination})}\nonumber\\
    &\leq 2 \log_2 K \int_{\bmu} \max_{r\in [R]}
    \exp\left(-g_1(i_*,j_{r,\bmu})\right) h(\bmu) \,d\bmu \nonumber\\
    &\leq 2\log_2 K \sum_{i\in[K]}\sum_{j>i} \int_{\bmu}
    \exp\left(-g_1(i,j)\right) h(\bmu) \,d\bmu ,
\end{align}  
Now, as noted before, $i_*$ and $j_{r,\bmu}$ are random quantities that depend on the instance $\bmu$. The third inequality is due to taking the maximum over $r\in[R]$ and then using the fact that for any $i,j$ we have $g_r(i,j)\geq g_1(i,j)$. 
For the last inequality, note that for any fixed $r\in[R]$, we can replace the (stochastic) values of $i_*$ and $j_{r,\bmu}$ by taking all possible pairs in $[K]$. Finally, we compute the integral in \cref{eq:regret_bound_1} by completing the squares and using properties of Gaussian integration:
\begin{restatable}{proposition}{lemFinalUBound}\label{lem:final_ub}
For any $i,j\in[K]$ we have that
\begin{align*}
    &\int_{\bmu}
    \exp\left(-g_1(i,j)\right) h(\bmu) \,d\bmu  =  \sqrt{\frac{R\sum_{k\in[K]}\s_k^2}{n\s_0^2+ R\sum_{k\in[K]}\s_k^2}} \exp\left(-\frac{(\nu_j-\nu_i)^2}{4\s_0^2}\right).
\end{align*}
\end{restatable}
\noindent This completes the sketch.
\qed

\section{Lower Bounds}\label{sec:lb'}

In this section, we give a distribution-dependent lower bound on the expected probability of misidentification. We focus on the important case of $2$-armed bandits and explore the dependence on the budget and prior quantities. We note that even in frequentist BAI, the case of $2$-armed bandits with Gaussian rewards is the only setting with matching upper and lower bounds \citep{KaufmannEmilie2016OnTC,Kato22,qin22a}. Our lower bound is the first such result in Bayesian BAI. To derive it, we consider symmetric instances $\bmu=(\mu_1,\mu_2)$ and $\bmu'=(\mu_1,\mu_2)$, and modify a change-of-measure argument in \citet{Carpentier16}, along with properties of Gaussian integration. We show that the guarantee of \bayeselim matches the lower bound (up to logarithmic factors) for any budget $n$. 
Finally, we present a lower bound on the expected probability of misidentification of any frequentist policy, a policy that ignores prior information, applied to the Bayesian setting. All omitted proofs of this section are in \cref{app:LB'}.

\begin{restatable}{theorem}{thmlb}\label{thm:lb'}
For any policy interacting with $K=2$ arms with mean rewards $\bmu=(\mu_1,\mu_2)$, reward distributions $\cN(\mu_i,\s^2)$, there exists prior $\mu_i\sim\mathcal{N}(\nu_i,\s_0^2)$ for $i\in\{1,2\}$, such that
\begin{align*}
    &\Ex{H}{\Prob{}{J\neq i_*(\bmu)|\bmu}}\geq \frac{1}{2e}\sqrt{\frac{\sigma^2 }{2n(8\log(2n)+1)\s_0^2+\s^2}} \exp\left(-\frac{(\nu_1-\nu_2)^2}{4\s_0^2}\right)
\end{align*}
\end{restatable} 
The above bound exhibits an $\Tilde{\Omega}({1}/{\sqrt{n}})$ dependence on the budget, while the dependence in prior quantities exactly matches that of our upper bound in \cref{thm:bound_with_integral}. 
In particular, the guarantee obtained for the expected probability of misidentification of \bayeselim in the case of $K=2$ arms becomes as follows.
\begin{corollary}
In setting of \cref{thm:lb'}, \bayeselim satisfies
    \begin{align*}
    &\Ex{H}{\Prob{}{J\neq i_*|\bmu}} \leq 2 \sqrt{\frac{2\s^2}{n\s_0^2+2\s^2}}  \exp\left(-\frac{(\nu_1-\nu_2)^2}{4\s_0^2}\right).
\end{align*} 
\end{corollary}
Therefore, \bayeselim achieves {optimal} guarantees (up to logarithmic factors) for any budget.
Finally, in the general setting of $K$ arms, we formally quantify the loss incurred by using a frequentist policy in the Bayesian setting. We show the following lower bound on the expected probability of misidentification of policies that ignore prior information.
\begin{restatable}{theorem}{thmFreqLBound}\label{thm:freq_lb'}
Let some $i\in[K]$. For any frequentist algorithm, i.e. any algorithm that is oblivious to prior information, there exist prior means $(\nu_1,\dots,\nu_K)$ with  $\nu_i>\nu_{j}$ for all $j\in [K]\setminus{\{i\}}$, such that, when $\s_0\rightarrow0$, the expected probability of misidentification satisfies
\begin{align*}
    &\Ex{H}{\prob{J\neq i_*|\bmu}} \geq \frac{1}{6} 
    \exp\left(-\frac{12n}{\sum_{j\neq i}(\nu_j-\nu_i)^{-2}} - \sqrt{\frac{48n\log(6Kn)}{\sum_{j\neq i}(\nu_j-\nu_i)^{-2}}}\right).
\end{align*}
\end{restatable}

Note that the setting of \cref{thm:freq_lb'} corresponds to a trivial case for a Bayesian policy, since the means are known with certainty, without any sampling. The probability of \bayeselim making a mistake is $0$ in this case.

\section{Experiments}
We conduct both synthetic experiments and an experiment on a real-world dataset \citep{hoogenboom2019dssat}.

\label{sec:experiments}
\subsection{Synthetic Data Experiments}

\label{subsec: synthetic experiments}
\begin{figure*}[t]
    \centering
    \begin{minipage}{.47\textwidth}
        \includegraphics[width=\linewidth]{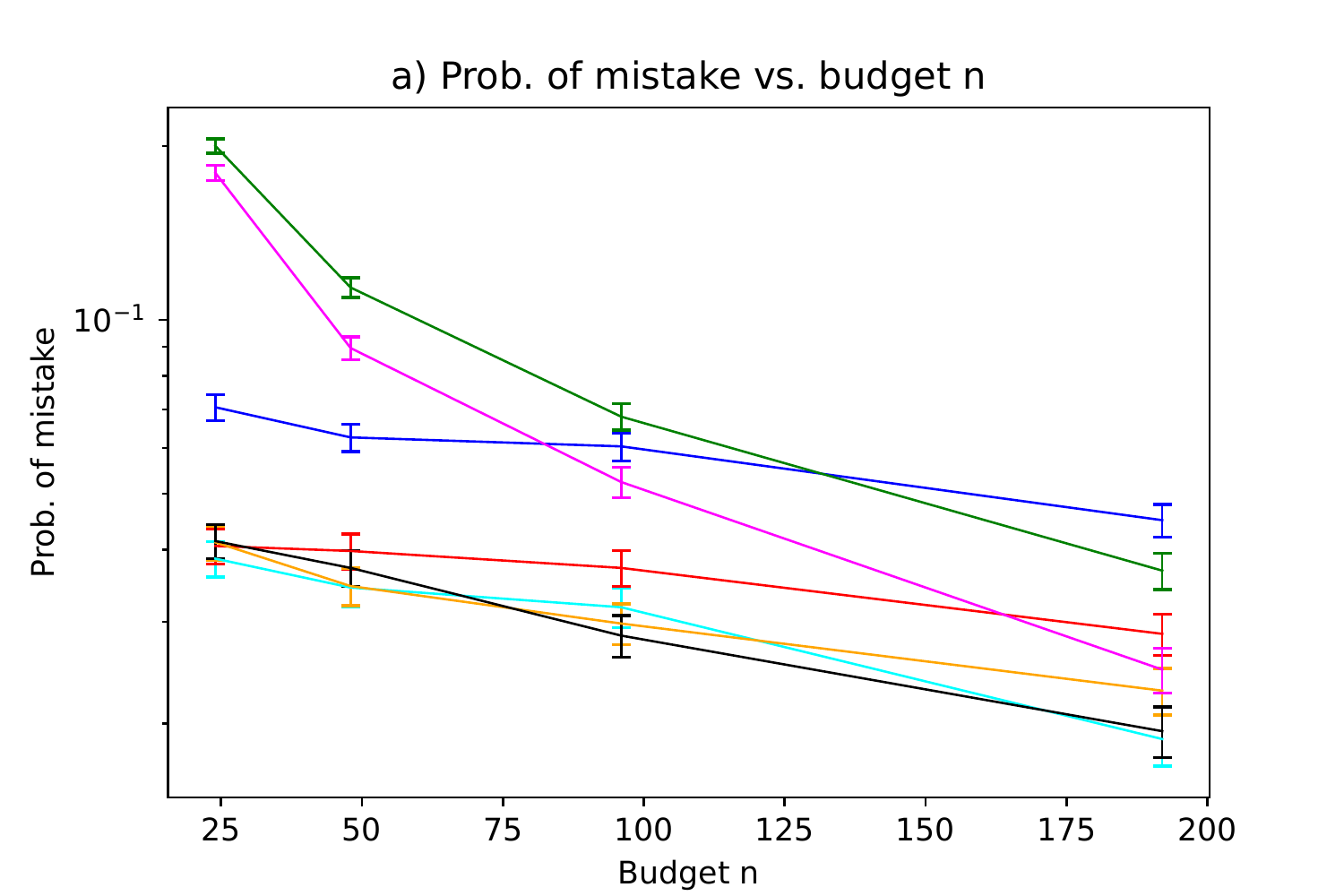}
    \end{minipage}
    \begin{minipage}{.47\textwidth}
        \includegraphics[width=\linewidth]{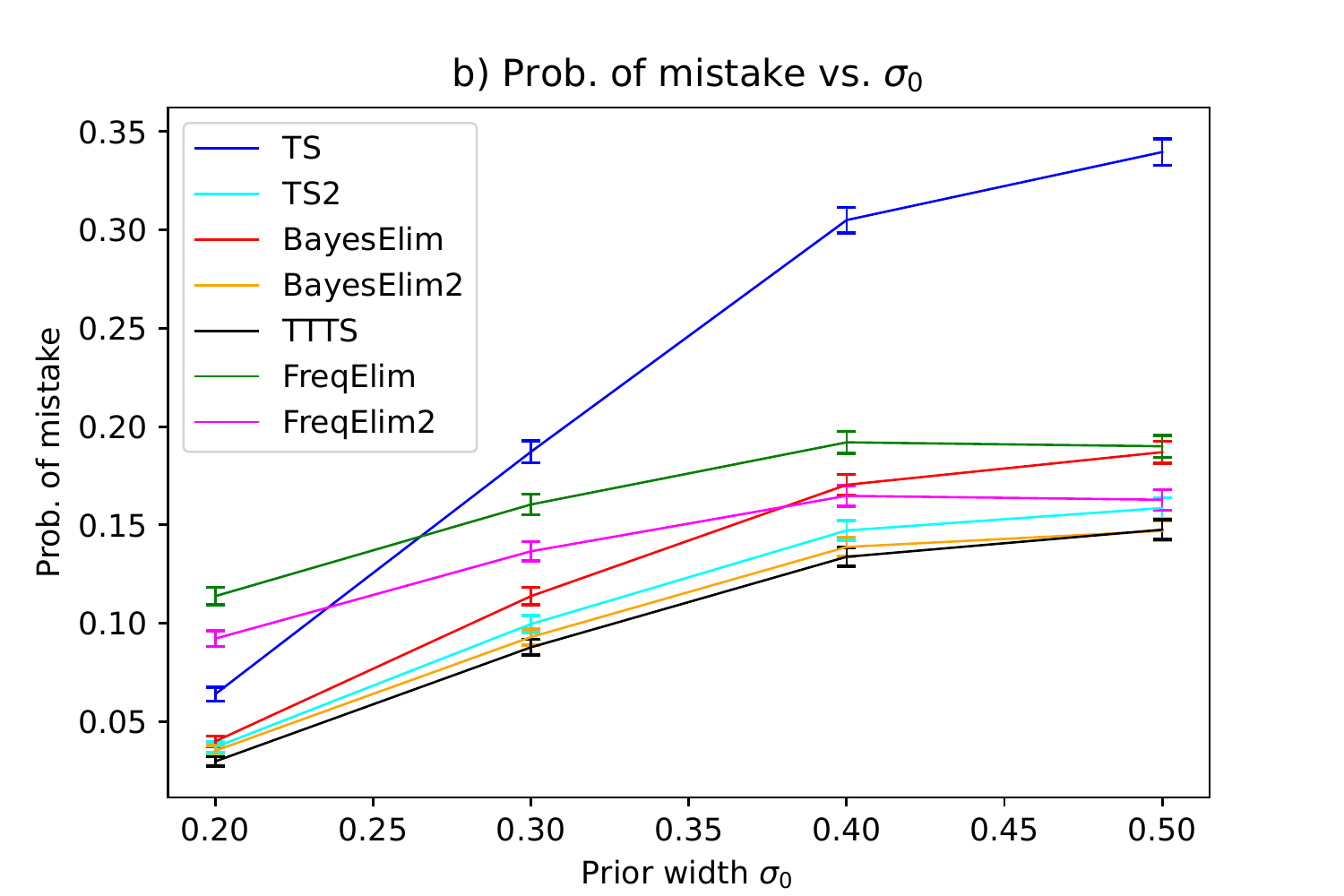}
    \end{minipage} 
    \caption{Evaluation of fixed-budget BAI algorithms on a synthetic dataset. The probability of misidentification as a function of a) budget $n$ and b) prior width $\sigma_0$.} 
    \label{fig:plot budget sigma0}
\end{figure*}

We simulate rewards from $K=8$ Gaussian arms whose means are drawn from a Gaussian prior with means $\{\nu_i = 2^{-i}: i \in \{0, 1, \dots, 7\}\}$. We choose this setting because it contains few small gaps and many large gaps. This is suitable for adaptive algorithms, as can be seen from \cref{thm:freq_bound}. In all our plots, we show the mean performance and error bars over $5000$ runs. We evaluate $7$ algorithms. We say that the algorithm \emph{does not have a theoretical guarantee} if it does not have a finite-budget Bayesian bound, on either its simple regret or misidentification probability.

\begin{itemize}
    \item \ts\citep{russo2018tutorial}: Thompson sampling where the best arm is chosen proportionally to the number of pulls. The simple regret of this variant of \ts is $\tilde{O}(1/\sqrt{n})$ (Section 9.1 in \citet{DBLP:journals/ior/RussoR18}).
    \item \tstwo: Thompson sampling where the best arm is the one with the highest posterior mean. This strategy does not have theoretical guarantees but performs very well empirically. 
    \item \bayeselim: Our proposed approach in \cref{alg:bayesian_successive_elimination}. Its probability of misidentification is bounded in \cref{thm:bound_with_integral}.
    \item \bayeselimtwo: A variant of \cref{alg:bayesian_successive_elimination} where we do not discard observations at the end of each round. It is well known that discarding of earlier observations is needed for analyzing BAI but it hurts practical performance \citep{ActionElimination-Evendar2006a,Karnin2013AlmostOE}. Thus this approach does not have a theoretical guarantee.
    \item \ttts \citep{russo2016simple}: Top-two Thompson sampling, a state-of-the-art algorithm for BAI. The arm with the highest posterior mean is chosen as the best arm. \ttts does not have a finite-budget guarantee. Its bound is asymptotic (Theorem 1 in \citet{russo2016simple}).
    \item \freqelim\citep{Karnin2013AlmostOE}: A frequentist variant of \cref{alg:bayesian_successive_elimination} studied in \cref{sec:frequentist bai}. Its probability of misidentification is bounded in \cref{thm:freq_bound}.
    \item \freqelimtwo: A frequentist analog of \bayeselimtwo. It does not have a theoretical bound.
\end{itemize}

In our first experiment, we study the dependence of the probability of misidentification on budget $n$. We fix the prior width at $\sigma_0 = 0.5$ and the reward noise at $\sigma = 0.5$. Our results are reported in \cref{fig:plot budget sigma0}a. As expected, the probability of misidentification decreases as the budget $n$ increases for all algorithms. We make two additional observations. First, \bayeselim performs the best among the algorithms with theoretical guarantees. Second, \bayeselimtwo, along with \ttts, are the best performing algorithms. While these algorithms do not have finite-budget Bayesian guarantees, we compare with them and report their performance in the hope that this may inspire future analyzable designs. Finally, the difference between the frequentist and Bayesian elimination algorithms decreases as the budget increases. This is expected since the benefit of the prior diminishes as more samples are available. 

In our second experiment, we study the dependence of the probability of misidentification on prior width $\sigma_0$. Our results are reported in \cref{fig:plot budget sigma0}b. We observe again that \bayeselim is the best algorithm among those with theoretical guarantees, and \bayeselimtwo is among the best algorithms. We also note that the performance gap between frequentist and Bayesian elimination algorithms decreases as the prior width increases. This is expected because wider priors are less informative.

\subsection{Real Data Experiments}

\begin{figure*}[t]
    \centering
    \begin{minipage}{.47\textwidth}
        \includegraphics[width=\linewidth]{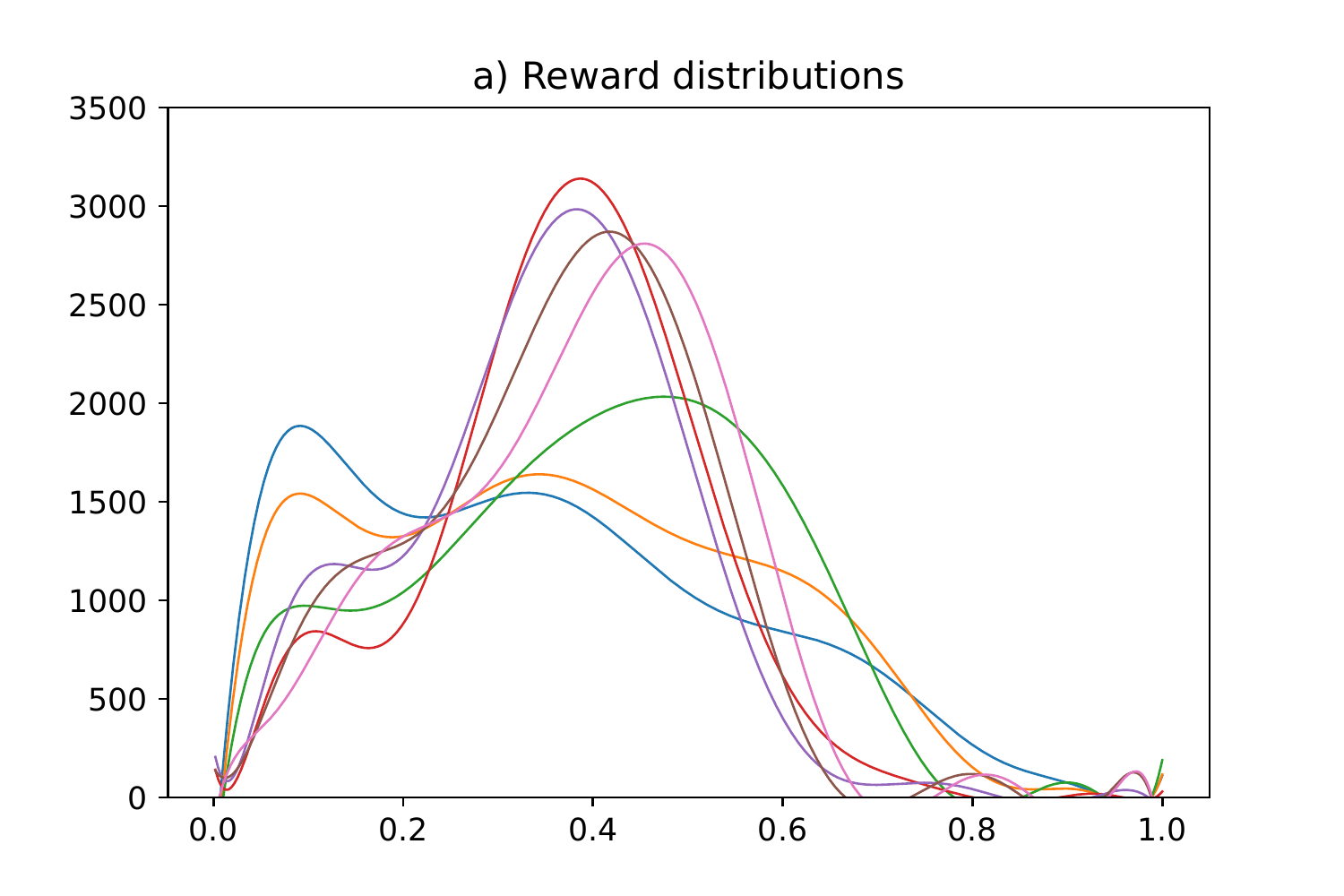}
    \end{minipage}
    \begin{minipage}{.47\textwidth}
        \includegraphics[width=\linewidth]{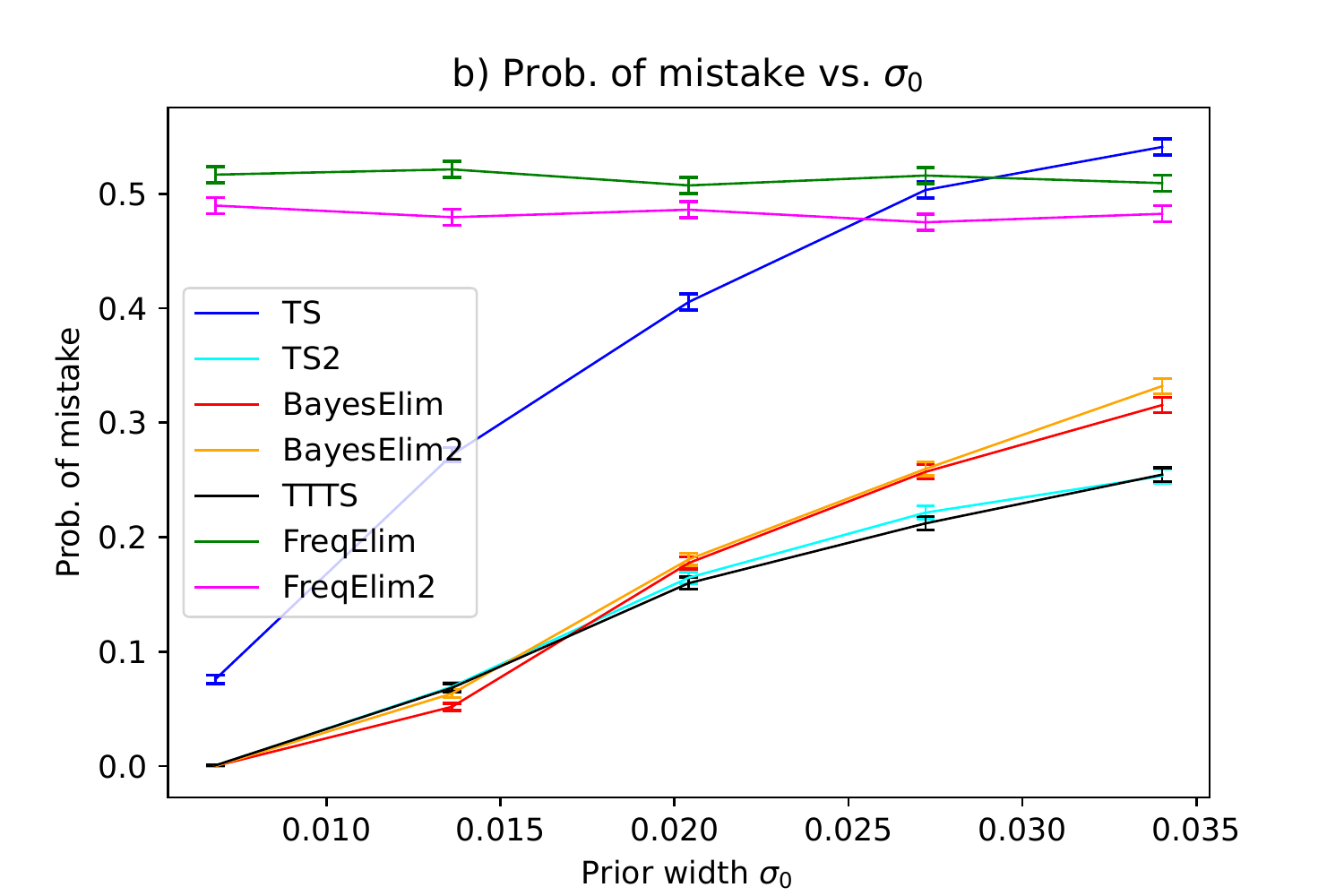}
    \end{minipage}
    \caption{a) Reward distributions of the $K=7$ arms in the crop management problem. b) Probability of misidentification as a function of the prior width $\sigma_0$.}
    \label{fig:crop data results}
\end{figure*}

For our real-data experiment, we consider the crop-management problem often used to compare bandit algorithms \citep{baudry2021optimal, jourdan2022top}, where a group of farmers wants to identify the best planting date for a rainfed crop. The reward (crop yield) can be modeled as a complex distribution with multiple modes, but upper bounded by a known maximum potential. 

We compare algorithms using the DSSAT simulator \citep{hoogenboom2019dssat}. Since calling the simulator is computationally intensive, we use the techniques and empirical data shared in \citet{jourdan2022top} (see Appendix I.1 in their paper for details) to run our experiments. There are $K=7$ arms, and each arm corresponds to a choice of planting date and fixed soil conditions. The reward distribution of the arms are plotted in \cref{fig:crop data results}a. 

Similarly to \cref{subsec: synthetic experiments}, we compare the algorithms by plotting their probability of misidentification as a function of the prior width $\sigma_0$ in \cref{fig:crop data results}b. We set the budget to $n=420$. We set the priors to be Gaussians with mean $\nu_i$ equal to the reward distribution mean from \cref{fig:crop data results}a, and study problems of increasing difficulty by varying $\sigma_0$ as follows. Let $\eta = \tfrac{\nu_1-\nu_2}{2}$ be the difference between the top two prior means. We vary $\sigma_0$ from $\eta$ to $5\eta$. When $\sigma_0 = \eta$ the problem is very easy; one can almost always correctly guess the best arm as the one with prior mean $\nu_1$ without observing any rewards. Even though our algorithms assume Gaussian rewards to update the posterior, the actual rewards are drawn from the multimodal distributions in \cref{fig:crop data results}a. Thus this experiment also tests the robustness of our algorithms to the Gaussian assumption. The probability of misidentification of all Bayesian algorithms increases with $\sigma_0$ in \cref{fig:crop data results}b as expected. We observe that \bayeselim has the lowest misidentification probability among all algorithms with theoretical guarantees. \tstwo and \ttts have the lowest misidentification probabilities, and are closely trailed by \bayeselim and \bayeselimtwo.

\section{Conclusion}
\label{sec:conclusions}

While best-arm identification with a fixed-budget has been studied extensively in the frequentist setting, Bayesian algorithms with optimal finite-budget guarantees do not exist. In this work, we address this gap and propose a Bayesian successive elimination algorithm with such a guarantee. The key idea in the algorithm is to eliminate arms based on their MAP estimates of mean rewards, which take the prior distribution of arm means into account. Our upper bound shows that the performance of the algorithm improves when the prior is more informative and matches our newly established lower bound for $K = 2$ arms. Our algorithm is evaluated empirically on synthetic and real-world datasets. We observe that it is superior to frequentist methods and competitive with state-of-the-art Bayesian algorithms with no guarantees in our setting.

Our work is the first step in the exciting direction of more sample-efficient Bayesian BAI algorithms, which have improved guarantees for more informative priors. The work can be extended in two obvious directions. First, our algorithm is designed for and analyzed in Gaussian bandits. We believe that both can be extended to single-parameter exponential-family distributions with conjugate priors, such as Bernoulli rewards with beta priors. Second, successive elimination of \citet{Karnin2013AlmostOE} has recently been extended to linear models by \citet{ijcai2022p388}. In the linear model, a Gaussian model parameter prior with Gaussian rewards implies a Gaussian model parameter posterior. For this conjugacy, we believe that our algorithm design and analysis can be extended to linear models.

\bibliographystyle{abbrvnat}
\bibliography{references}

\newpage

\appendix

\newpage
\section{Appendix}
\subsection{Table of Notation}\label{app:notation_tab}
\begin{table}[h]
\centering
\begin{tabular}{ |p{0.09\textwidth}|p{0.34\textwidth}| |p{0.09\textwidth}|p{0.34\textwidth}|  }
 \hline
 \multicolumn{4}{|c|}{Notation} \\
 \hline
 $K$ & Number of arms & $n$ & Exploration budget \\
 $J$ & Arm recommended by the policy &
 $\bmu$ & Mean reward vector \\
 $X_i$ &  Stochastic reward of arm $i$ &
 $\mu_i,\s_i^2$  & Mean and variance of the reward distribution of arm $i$ \\
 $\nu_i,\s_{0}^2$ & Mean and variance of the prior distribution of arm $i$  &
 $\bar \mu_{i,m}, \bs_{i,m}^2$ & Posterior mean and variance of arm $i$ computed from $m$ i.i.d. samples \\
 $i_*,i_*(\bmu)$ & Optimal arm of a random instance $\bmu$ &
 $\nu_*, \mu_*$ & Prior and mean reward of arm $i_*$ \\
 $R$ & Number of elimination rounds &
 $S_r$ & Active set of arms in elimination round $r$ \\
 $\prob{\cdot|\bmu}$ & Probability measure considering a fixed mean reward vector $\bmu$ &
 $\Ex{H}{.}$ & Expectation over the randomness of mean rewards \\
 $H(\bmu)$ & Cumulative probability of instance $\bmu$ &
 $h(\bmu)$ & Probability density of $\bmu$ \\
 \hline
\end{tabular}
\caption{Notation used in the paper.}
\label{table:1}
\end{table}

\subsection{Cumulative Regret, Simple Regret, and Probability of Misidentification} \label{sec:pure exploration}

Although probability of misidentification and simple regret are different objectives, there are relations between them, as well as with the objective of cumulative regret. The simple regret of a policy that recommends arm $J$ in a bandit instance $\bmu$ is defined as $ r(\bmu)=\sum_{i\in[K]\setminus\{i_*(\bmu)\}} \Delta_i \prob{J= i|\bmu}$.
% and best-arm identification are widely studied by the same community
In particular, in the frequentist setting and for $[0,1]$-bounded rewards, error probability and simple regret trivially satisfy \citep{audibert-2010-BAI} 
\begin{align*}\Delta_{\min} \prob{J\neq i_*(\bmu)|\bmu}\leq  r(\bmu) \leq \prob{J\neq i_*(\bmu)|\bmu}.\end{align*} 
This relation does not hold for general reward distributions with possibly unbounded gaps. In addition, there are known relations between cumulative regret and simple regret, which also extend to the Bayesian setting. Specifically, a $O(1 / \sqrt{n})$ simple Bayes regret can be attained by a trivial reduction from an algorithm with a $O(\sqrt{n})$ cumulative Bayes regret (Section 9.1 in \citet{DBLP:journals/ior/RussoR18}).

Moreover, although the objectives of error probability and simple regret are different, due to their closeness, works focusing on one objective can provide a lot of intuition for the other. 
For instance, in the recent work of \citet{komiyama2021optimal}, the authors show that, for Bernoulli rewards and under mild continuity assumptions on the prior, the area where the mean gaps are of order $O\left(\sqrt{{\log n}/{n}}\right)$ contributes a  $O\left({1}/{n}\right)$ to the expected simple regret as $n\rightarrow\infty$. If we adjust their calculations to the objective of error probability, it can be implied that, for Bernoulli rewards, the same area of gaps contributes an $O\left({1}/{\sqrt{n}}\right)$ to the expected probability of misidentification when $n\rightarrow\infty$. 

Despite the above similarities, which are based on informal arguments, we make several contributions that set us apart from \citet{komiyama2021optimal}: (i) We provide finite-budget guarantees on the probability of misidentification, which match the lower bound for any fixed budget. The analysis in \citet{komiyama2021optimal} is asymptotic and their guarantees on simple regret only match the lower bound when the budget $n \to \infty$. (ii) We achieve our guarantees by proposing a novel Bayesian algorithm for minimizing the expected probability of misidentification. Such guarantees cannot be achieved by any frequentist algorithm, such as that of \citet{Karnin2013AlmostOE} for minimizing the probability of misidentification or that of \citet{komiyama2021optimal} for simple regret minimization. In particular, we show in \cref{thm:freq_lb'} that for any fixed budget, the guarantees achieved by any frequentist algorithm can be arbitrarily worse than those of \bayeselim. (iii) Finally, we focus on Gaussian rewards instead of Bernoulli. \citet{komiyama2021optimal} leave open the question of deriving asymptotic guarantees for Gaussian rewards.

\subsection{Allocation Rule of \cref{eq:split}}\label{app:allocation}
In \cite{KaufmannEmilie2016OnTC} the authors focus on the important case of Gaussian $2$-armed frequentist BAI and identify an allocation rule such that the probability or error is minimized. In their case, the probability of error corresponds to the probability that the empirical mean of the suboptimal arm is larger than the empirical mean of the optimal arm.
In this paper, due to the fact that the objective is different than that of frequentist BAI and due to the presence of prior information, we diverge from the frequentist methods in our selection of an allocation rule. The difference between the frequentist approach to ours becomes more clear in the following. 

Focusing on $K=2$ arms and considering a fixed parameter vector $\bmu$, we examine the probability that the posterior mean of the suboptimal arm is larger than the posterior mean of the optimal arm. 
% The posterior distribution of any arm given $n_{i}$ samples, $X_{i,1},...,X_{i,n_{i}}$, is 
% $\mathcal{N}\left(\bar\mu_{i,n_{i}},\bar\s_{i,n_{i}}^2\right)$
% where 
% $$\bar\s_{i,n_{i}}^2 = \left(\frac{1}{\frac{1}{\s_{0,i}^2}+\frac{n_{i}}{\s_i^2}}\right) 
% \text{ and  }~ 
% \bar\mu_{i,n_{i}} = \bar\s_{i,n_{i}}^2 \left(\frac{\nu_i}{\s_{0,i}^2}+\sum_{s\in [n_{i}]} \frac{X_{i,s}}{\s_i^2}\right).$$
Let $i_*$ and $i$ be the optimal and suboptimal arm of $\bmu$. The probability of error for fixed $\bmu$ is
\begin{align*}
    &\prob{\bar\mu_{i,n_{i}} > \bar\mu_{i_*,n_{i_*}}|\bmu} 
    = \prob{ \bar\s_{i,n_{i}}^2 \left(\frac{\nu_i}{\s_{0,i}^2}+\sum_{s\in [n_{i}]} \frac{X_{i,s}}{\s_i^2}\right) > \bar\s_{i_*,n_{i_*}}^2 \left(\frac{\nu_{i_*}}{\s_{0,i_*}^2}+\sum_{s\in [n_{i_*}]} \frac{X_{i_*,s}}{\s_{i_*}^2}\right) ~\bigg|\bmu}\\
    &= \prob{ \bar\s_{i,n_{i}}^2\sum_{s\in [n_{i}]} \frac{X_{i,s}}{\s_i^2} - \bar\s_{i_*,n_{i_*}}^2\sum_{s\in [n_{i_*}]} \frac{X_{i_*,s}}{\s_{i_*}^2}> \frac{\bar\s_{i_*,n_{i_*}}^2\nu_{i_*}-\bar\s_{i,n_{i}}^2\nu_i}{\s_{0}^2} ~\bigg|\bmu}  \\
    &= \mathbb{P}\Bigg( \bar\s_{i,n_{i}}^2\sum_{s\in [n_{i}]} \frac{X_{i,s}-\mu_i}{\s_i^2} - \bar\s_{i_*,n_{i_*}}^2\sum_{s\in [n_{i_*}]} \frac{X_{i_*,s}-\mu_{i_*}}{\s_{i_*}^2} > \frac{\bar\s_{i_*,n_{i_*}}^2\nu_{i_*}-\bar\s_{i,n_{i}}^2\nu_i}{\s_{0}^2} + \bar\s_{i_*,n_{i_*}}^2\frac{n_{i_*}\mu_{i_*}}{\s_{i_*}^2} -  \bar\s_{i,n_{i}}^2\frac{n_{i}\mu_{i}}{\s_{i}^2} ~\bigg|\bmu\Bigg). 
\end{align*}
The LHS of the above expression is a zero mean Gaussian random variable. In contrast to the analysis in \cite{KaufmannEmilie2016OnTC} where the authors use Hoeffding's inequality in order to optimize the probability bound, in our case, due to the bias term of the prior $\frac{\bar\s_{i_*,n_{i_*}}^2\nu_{i_*}-\bar\s_{i,n_{i}}^2\nu_i}{\s_{0}^2}$, the RHS can become negative. Moreover, since our objective is different, we are not looking to minimize the RHS for a single mean vector, but on expectation over the prior. 

We introduce the following heuristic: we select an allocation rule such that the posterior variances of the arms become equal. In other words, we impose the following condition
\begin{align*}
    \bar\s_{i,n_{i}}^2 = \bar\s_{i_*,n_{i_*}}^2.
\end{align*}
The above condition implies that $
n_{r,i}=n\frac{\s_i^2}{\s_i^2+\s_{i_*}^2}$,
which coincides with our allocation rule of \cref{eq:split} for $K=2$ arms. In \cref{sec:ub_analysis}, even for $K>2$ arms, we show that our algorithm, using the above heuristic, has an optimal guarantee (up to logarithmic terms).

\subsection{Auxilliary}
\begin{theorem}[Hoeffding's Inequality for Subgaussian Random Variables (from \cite{hoeffding63})]\label{thm:hoeffding}
If $X_1,...,X_m\sim \cN(\mu,\s^2)$ then for any $i\in[m]$:
\begin{align*}
    &\Prob{}{X_i\geq \mu+\epsilon} \leq \exp\left(-\frac{\epsilon^2}{2\s^2}\right) 
    \text{  and  } \Prob{}{\frac{1}{m}\sum_{i\in[m]}X_i\geq \mu+\epsilon} \leq \exp\left(-\frac{m\epsilon^2}{2\s^2}\right).
\end{align*}
\end{theorem}

\begin{theorem}[Bretagnolle–Huber inequality (from \cite{Bretagnolle1978EstimationDD})]\label{thm:huber}
Let $\mathbb{P}$ and $\mathbb{Q}$ be probability
measures on the same measurable space and $A$ a measurable event. Then,
\begin{align*}
    \mathbb{P}(A) + \mathbb{Q}(A^c) \geq \frac{1}{2} \exp(-d_{KL}(\mathbb{P},\mathbb{Q}))
\end{align*}
where $A^c$ is the complement of $A$ and $d_{KL}(\mathbb{P},\mathbb{Q}) = \int_{-\infty}^{\infty} \log \left(\frac{\,d \mathbb{P}(x)}{\,d \mathbb{Q}(x)}\right) \,d \mathbb{P}(x)$.
\end{theorem}

\section{Proofs of \cref{sec:ub_analysis}}\label{app:UB}

\lemPostManConc*
\begin{proof}
% Let
% \begin{align*}
%     h(i,j) = -n\frac{(\mu_{i}- \mu_j)}{4R\cdot \sr }\left(\frac{R\cdot\sr}{n\s_0^2}(\nu_{i}-\nu_j) + \mu_{i}- \mu_j\right).
% \end{align*}
We have that:
\begin{align*}
    g_r(i,j) =n\frac{(\mu_{i}- \mu_j)^2}{4R \sum_{k\in S_r}\s_k^2 }+\frac{(\nu_{i}-\nu_j)(\mu_{i}- \mu_j)}{2\s_0^2}.
\end{align*}
Let $S_r = \sum_{k\in S_r}\s_k^2$.
The posterior distribution of any arm $i$ given $n_{r,i}$ samples, $X_{i,1},...,X_{i,n_{r,i}}$, is 
$\mathcal{N}\left(\bar\mu_{i,n_{r,i}},\bar\s_{i,n_{r,i}}^2\right)$
where 
$$\bar\s_{i,n_{r,i}}^2 = \left(\frac{1}{\frac{1}{\s_{0,i}^2}+\frac{n_{r,i}}{\s_i^2}}\right) 
\text{ and  }~ 
\bar\mu_{i,n_{r,i}} = \bar\s_{i,n_{r,i}}^2 \left(\frac{\nu_i}{\s_{0,i}^2}+\sum_{s\in [n_{r,i}]} \frac{X_{i,s}}{\s_i^2}\right).$$
% \begin{proposition}
% At each round $r$, for the choice of $n_{r,i}$ for $i\in[K]$ of \cref{alg:bayesian_successive_elimination} we have that $\bar\s_{i,n_{r,i}}^2 = \bar\s_{j,\nr{j}}^2$ for all $i,j\in S_r$.
% \end{proposition}
% \begin{proof}
% TBD
% \end{proof}
We consider the probability that the posterior mean of some arm $i\in S_r$ is larger than the posterior mean of arm $i_*$ for a fixed parameter vector $\bmu$. We have that:
\begin{align}\label{eq:concentration_2}
    &\prob{\bar\mu_{i,n_{r,i}} > \bar\mu_{i_*,n_{r,i_*}}|\bmu} \\
    &= \prob{ \bar\s_{i,n_{r,i}}^2 \left(\frac{\nu_i}{\s_{0,i}^2}+\sum_{s\in [n_{r,i}]} \frac{X_{i,s}}{\s_i^2}\right) > \bar\s_{i_*,n_{r,i_*}}^2 \left(\frac{\nu_{i_*}}{\s_{0,i_*}^2}+\sum_{s\in [n_{r,i_*}]} \frac{X_{i_*,s}}{\s_{i_*}^2}\right) ~|\bmu} \nonumber \,.  
\end{align}
Notice that for the values of $n_{r,i}$ selected by our Algorithm, we have that ${\bar\s_{i,n_{r,i}}^2} = {\bar\s_{j,n_{r,j}}^2}$ for all $i,j\in S_r$, and 
\begin{align}\label{eq:tmp_s3}
    \frac{ n_{r,i}}{\s_i^2} 
    &= {\s_i^2}\frac{n}{R\cdot\sr} \cdot \frac{1}{\s_i^2}  
    = \frac{n}{R\cdot\sr} .
\end{align}
Therefore, for \cref{eq:concentration_2} we have that:
\begin{align}\label{eq:tmp_s5}
    &\prob{\bar\mu_{i,n_{r,i}} > \bar\mu_{i_*,n_{r,i_*}}|\bmu} 
    = \prob{ \bar\s_{i,n_{r,i}}^2 \left(\frac{\nu_i}{\s_{0}^2}+\sum_{s\in [n_{r,i}]} \frac{X_{i,s}}{\s_i^2}\right) > \bar\s_{i_*,n_{r,i_*}}^2 \left(\frac{\nu_{i_*}}{\s_{0}^2}+\sum_{s\in [n_{r,i_*}]} \frac{X_{i_*,s}}{\s_{i_*}^2}\right) ~|\bmu} \nonumber \\
    &= \prob{ \frac{\nu_i}{\s_{0}^2}+\sum_{s\in [n_{r,i}]} \frac{X_{i,s}}{\s_i^2}> \frac{\nu_{i_*}}{\s_{0}^2}+\sum_{s\in [n_{r,i_*}]} \frac{X_{i_*,s}}{\s_{i_*}^2} ~|\bmu} \nonumber \\
    &= \prob{ \sum_{s\in [n_{r,i}]} \frac{X_{i,s}}{\s_i^2} - \sum_{s\in [n_{r,i_*}]} \frac{X_{i_*,s}}{\s_{i_*}^2}> \frac{\nu_{i_*}-\nu_i}{\s_{0}^2} ~|\bmu} \nonumber \\
    &= \prob{ \sum_{s\in [n_{r,i}]} \frac{X_{i,s}-\mu_i}{\s_i^2} - \sum_{s\in [n_{r,i_*}]} \frac{X_{i_*,s}-\mu_{i_*}}{\s_{i_*}^2}> \frac{\nu_{i_*}-\nu_i}{\s_{0}^2} + \frac{n_{r,i_*}\mu_{i_*}}{\s_{i_*}^2} -  \frac{n_{r,i}\mu_{i}}{\s_{i}^2} ~|\bmu} \nonumber \\
    &= \prob{ \sum_{s\in [n_{r,i}]} \frac{X_{i,s}-\mu_i}{\s_i^2} - \sum_{s\in [n_{r,i_*}]} \frac{X_{i_*,s}-\mu_{i_*}}{\s_{i_*}^2} > \frac{\nu_{i_*}-\nu_i}{\s_{0}^2} + \frac{n}{R\cdot\sr} (\mu_{i_*} - \mu_i) ~|\bmu} 
\end{align}

Since $X_{i,s}\sim\cN(\mu_i,\s_i^2)$, the sum
$\frac{1}{\s_i^2}\sum_{s\in [n_{r,i}]} (X_{i,s}-\mu_i)$
is a zero mean Gaussian random variable with variance:
\begin{align*}
    \frac{1}{\s_i^4} \cdot n_{r,i} \cdot \s_i^2 
    = \frac{1}{\s_i^4} \frac{n\cdot \s_i^2}{R\sr} \cdot \s_i^2 
    = \frac{n}{R\cdot \sr}.
\end{align*}
Therefore, the difference on the LHS of the above inequality:  $\sum_{s\in [n_{r,i}]} \frac{X_{i,s}-\mu_i}{\s_i^2} - \sum_{s\in [n_{r,i_*}]} \frac{X_{i_*,s}-\mu_{i_*}}{\s_{i_*}^2}$, is a zero mean Gaussian with variance $2\frac{n}{R\cdot \sr}$.
Assume that $\frac{\nu_{i_*}-\nu_i}{\s_{0}^2} + \frac{n}{R\cdot\sr} \left(\mu_{i_*} - \mu_i\right)\geq 0$. 
Then, by Hoeffding's inequality (\cref{thm:hoeffding}) we get that:
\begin{align*} 
    \cref{eq:tmp_s5} 
    &\leq \exp\left(-\frac{\left(\frac{\nu_{i_*}-\nu_i}{\s_{0}^2} + \frac{n}{R\cdot\sr} \left(\mu_{i_*} - \mu_i\right)\right)^2 }{4\frac{n}{R\cdot \sr}}\right) \\
    &= \exp\left(-\frac{n}{4R\cdot \sr }\left(\frac{R\cdot\sr}{n\s_0^2}(\nu_{i_*}-\nu_i) + \mu_{i_*}- \mu_i\right)^2\right)\\
    &= \exp\left(-\frac{n}{4R\cdot \sr }\left((\mu_{i_*}-\mu_i)^2 + 2(\mu_{i_*}- \mu_i)\frac{R\cdot\sr}{n\s_0^2}(\nu_{i_*}-\nu_i) + \left(\frac{R\cdot\sr}{n\s_0^2}\right)^2(\nu_{i_*}-\nu_i)^2 \right)\right)\\
    &\leq \exp\left(-\frac{n}{4R\cdot \sr }(\mu_{i_*}-\mu_i)^2 - \frac{1}{2\s_0^2}(\mu_{i_*}- \mu_i)(\nu_{i_*}-\nu_i)  \right)\\
    &= \exp\left(-g_r(i_*,i)\right).
\end{align*} 
On the other hand, when $\frac{\nu_{i_*}-\nu_i}{\s_{0}^2} + \frac{n}{R\cdot\sr} (\mu_{i_*} - \mu_i)\leq 0$, since $\mu_{i_*} - \mu_i\geq 0$ then $\nu_{i_*}-\nu_i\leq0$ , and therefore we have $g_r(i_*,i)\leq 0$. Thus:
\begin{align*}
    \prob{\bar\mu_{i,n_{r,i}} > \bar\mu_{i_*,n_{r,i_*}}|\bmu} \leq \exp\left(-g_r(i_*,i)\right).
\end{align*}

\end{proof}

\lemRElimination*
\begin{proof}
% Let
% \begin{align*}
%     h(i,j) = -n\frac{(\mu_{i}- \mu_j)}{4R\cdot \sr }\left(\frac{R\cdot\sr}{n\s_0^2}(\nu_{i}-\nu_j) + \mu_{i}- \mu_j\right).
% \end{align*}
% We define $S'_r$ to be the set of $\frac{3|S_r|}{4}$ arms in $S_r$ with smallest posterior means. 
We have that: \begin{align}\label{eq:tmp_s6}
    \E{\sum_{i\in S_r\setminus\{i_*\}} \mathbb{I}\left((\bar\mu_{i,n_{r,i}} >\bar\mu_{i_*,n_{r,i_*}}) | \bmu, \{i_ \in S_r\}\right)} 
    &= \sum_{i\in S_r\setminus\{i_*\}} \prob{ \bar\mu_{i,n_{r,i}} > \bar\mu_{i_*,n_{r,i_*}} | \bmu, \{i_* \in S_r\}} \nonumber\\
    &\leq \sum_{i\in S_r\setminus\{i_*\}} \exp\left(-g_r(i_*,i)\right),
\end{align}
due to \cref{lem:posterior_mean_concentration}. 
Let some $j_{r,\bmu}\in S_r\setminus\{i_*\}$ such that for all $i\in S_r\setminus\{i_*\}$ we have $g_r(i_*,i)\geq g_r(i_*,j_{r,\bmu})$. We can upper bound \cref{eq:tmp_s6} as follows:
\begin{align}\label{eq:tmp_s7}
    \cref{eq:tmp_s6} 
    &= \sum_{i\in S_r\setminus\{i_*\}} \exp\left(-g_r(i_*,i)\right) \leq |S_r-1| \exp\left(-g_r(i_*,j_{r,\bmu})\right) .
\end{align}
Now, since the best arm is eliminated at the end of round $r$, then at least $\left\lceil\frac{|S_r-1|}{2}\right\rceil$ arms in $S_r\setminus\{i_*\}$ must have larger posterior means than $i_*$. Using this fact we get:
\begin{align*}
    &\prob{i_*(\bmu) \not \in S_{r+1}| \bmu, \{i_*(\bmu) \in S_r\}} \\
    &\leq \prob{ \sum_{i\in S_r \setminus\{i_*\}} \mathbb{I}( \bar\mu_{i,n_{r,i}} >\bar\mu_{i_*,n_{r,i_*}}) | \bmu, \{i_*(\bmu) \in S_r\} > \left\lceil\frac{|S_r-1|}{2}\right\rceil} & {} \\
    &\leq \frac{ \E{\sum_{i\in S_r\setminus\{i_*\}} \mathbb{I}( \bar\mu_{i,n_{r,i}} >\bar\mu_{i_*,n_{r,i_*}}) | \bmu, \{i_*(\bmu) \in S_r\}}}{\left\lceil\frac{|S_r-1|}{2}\right\rceil} &\text{(by Markov's inequality)}\\
    &\leq 2\exp\left(-g_r(i_*,j_{r,\bmu})\right). &\text{(by \cref{eq:tmp_s7})}
\end{align*}
\end{proof}

\lemFinalUBound*

\begin{proof}
The result follows by application of \cref{lem:gaussuan_integration_1} in the Appendix.
\end{proof}

\thmFreqUBound*

\begin{proof}
This proof follows the lines of \cref{thm:bound_with_integral}. To simplify the analysis, we ignore errors due to rounding. Note that the set of active arms in any round $r\in [R]$, i.e. $S_r$, is a random variable that depends on the reward realizations as well as the randomness in the parameters $\bmu$ of the reward distributions of the arms. We first consider the parameter vector $\bmu$ fixed. For any fixed round $r\in [R]$, \cref{lem:posterior_mean_concentration_freq} bounds the probability that the empirical mean of some suboptimal arm $i$ is larger that the empirical mean of the optimal arm of $\bmu$. 
% Recall that when $\s_i=\s$ for all $i\in[K]$ for some $\s>0$, the per-phase budget is distributed equally among active arms, i.e. $\nr{i}=\nr{}=\frac{n}{R|S_r|},\forall i\in S_r$. 
Let $\hat\mu_{i,n_{r,i}}$ be the empirical mean from $n_{r,i}$ samples of arm $i$. The following is an adaptation of Lemma $4.2$ of \citep{Karnin2013AlmostOE} for Gaussian reward distributions with different variances: 
\begin{lemma}[Adaptated Lemma $4.2$ of \citep{Karnin2013AlmostOE}]\label{lem:posterior_mean_concentration_freq}
Fix instance $\bmu$ and round $r \in [R]$. Suppose that $i_* \in S_r$. Then, for any $i \in S_r$:
\begin{align*}
    \prob{\hat\mu_{i,n_{r,i}}>\hat\mu_{i_*,n_{r,i_*}}|\bmu} \leq \exp\left(-\frac{1}{4\sum_{k\in S_r} \s_k^2}\nr{} \Delta_i^2\right) .
\end{align*}
\end{lemma}

Then, continuing along the lines of the frequentist proof, for any fixed $r\in[R]$ and $\bmu$, in \cref{lem:r_elimination_freq} we bound the probability that the optimal arm is eliminated at round $r$.

\begin{lemma}[Adaptated Lemma $4.3$ of \citep{Karnin2013AlmostOE}]\label{lem:r_elimination_freq}
Fix instance $\bmu$ and round $r\in [R]$. Then, there exists some $j_r\in S_r$ such that the probability that $i_*$ is eliminated at $r$ satisfies:
\begin{align*}
    \prob{i_* \not \in S_{r+1}| \{i_* \in S_r\},  \bmu}
    \leq 2\exp\left(-\frac{1}{4\sum_{k\in S_r} \s_k^2}\nr{} \Delta_{j_r}^2\right).
\end{align*} 
\end{lemma}

Up to this point we have considered a fixed reward vector and the above analysis imitates the frequentist analysis in this setting. Finally, we deal with the randomness in $\bmu$.

\begin{align*}
    \Ex{H}{\prob{J\neq i_*|\bmu}} 
    &= \int_{\bmu} \prob{J\neq i_*|\bmu} h(\bmu) \,d\bmu \nonumber\\
    &\leq \int_{\bmu} \sum_{r\in [R]}\prob{i_*(\bmu) \not \in S_{r+1}| \bmu, \{i_*(\bmu) \in S_r\}} h(\bmu) \,d\bmu \\
    &\leq 2\sum_{r\in [R]} \int_{\bmu}
    \exp\left(-\frac{n_r}{4 \sum_{k\in S_r} \s_k^2}\Delta_{j_r}^2\right)  \cdot h(\bmu) \,d\bmu , 
\end{align*}
where the last inequality is due to \cref{lem:r_elimination_freq}.
Now, as noted before, $i_*, j_r$ are random quantities that depend on the instance $\bmu$. For any fixed $r\in[R]$, we can rewrite the above integral by grouping the instances according to the realizations of $i_*$ and $j_r$ and then upper bound the integral to show the following: 

We have that:
\begin{align*}
    &\int_{\bmu}
    \exp\left(-\frac{n_r}{4 \sum_{k\in S_r} \s_k^2}\Delta_{j_r}^2\right)
    h(\bmu) \,d\bmu =\\
    &= \sum_{i\in[K]}\sum_{j>i}
    {\Bigg[}\int_{\bmu~:~ i_*=i,j_r=j}
    \exp\left(-\frac{n_r}{4\s^2}\left( \mu_i-\mu_j\right)^2\right)
    h(\bmu) \,d\bmu  \\
    &\quad +
    \int_{\bmu~:~ i_*=j,j_r=i}
    \exp\left(-\frac{n_r}{4\sum_{k\in S_r} \s_k^2}\left( \mu_j-\mu_i\right)^2\right)
    h(\bmu) \,d\bmu {\Bigg]}\\
    &\leq \sum_{i\in[K]}\sum_{j>i}
    {\Bigg[}\int_{\bmu~:~ \mu_i\geq \mu_j}
    \exp\left(-\frac{n_r}{4\sum_{k\in S_r} \s_k^2}\left( \mu_i-\mu_j\right)^2\right)
    h(\bmu) \,d\bmu  \\
    &\quad +
    \int_{\bmu~:~ \mu_i<\mu_j}
    \exp\left(-\frac{n_r}{4\sum_{k\in S_r} \s_k^2}\left( \mu_j-\mu_i\right)^2\right)
    h(\bmu) \,d\bmu {\Bigg]}\\
    &= \sum_{i\in[K]}\sum_{j>i}
    \int_{\bmu}
    \exp\left(-\frac{n_r}{4\sum_{k\in S_r} \s_k^2}\left( \mu_i-\mu_j\right)^2\right)
    h(\bmu) \,d\bmu \qquad\text{ (by symmetry)}\\
    &=  \frac{1}{\sqrt{n\frac{\s_0^2}{\sum_{k\in S_r} \s_k^2}\frac{2^r}{K\log_{2}(K)}+1}}  \sum_{i\in[K]}\sum_{j>i}\exp\left(-\frac{n\s_0^2}{n\s_0^2+\frac{K\log_2(K)}{2^r}\sum_{k\in S_r} \s_k^2}\frac{(\nu_i-\nu_j)^2}{4\s_0^2}\right) \,,
\end{align*}
where the last equality is due to  \cref{lem:gaussuan_integration2} in the Appendix. 
This completes the proof of \cref{thm:freq_bound}.

\end{proof}

\subsection{Gaussian Integrals}

\begin{lemma}\label{lem:gaussuan_integration_1}
Let $\mu_i\sim \cN(\nu_i,\s_{0}^2), \mu_j\sim \cN(\nu_j,\s_{0}^2)$ and $c_1,c_2\geq 0$. We have that:
\begin{align*}
    &\int_{\bmu} 
    \exp\left(-c_1 (\mu_i-\mu_j)^2 - c_2\frac{(\mu_i-\mu_j)(\nu_i-\nu_j)}{\s_0^2} \right)
    h(\bmu) \,d\bmu 
    = \dfrac{1}{\sqrt{4c_1{\sigma}_0^2+1}} \exp\left({-\frac{\left({\nu}_2-{\nu}_1\right)^2\left(c_1{\sigma}_0^2-c_2^2+c_2\right)}{{\sigma}_0^2\cdot\left(4c_1{\sigma}_0^2+1\right)}}\right)
\end{align*}
\end{lemma}
\begin{proof}
By sequentially grouping terms and computing the Gaussian integrals we get that:
\begin{align*}
    &\int_{\mu} \exp\left(-c_1 (\mu_i-\mu_j)^2 - c_2\frac{(\mu_i-\mu_j)(\nu_i-\nu_j)}{\s_0^2} \right) h(\bmu)\,d\bmu =\\
    &= \int_{\bmu} \frac{1}{2\pi\sigma_0^2} \exp\left({-c_1 (\mu_i-\mu_j)^2 - c_2\frac{(\mu_i-\mu_j)(\nu_i-\nu_j)}{\s_0^2}} \right)\exp\left({-\frac{(\mu_1-\nu_1)^2}{2\sigma_0^2}-\frac{(\mu_2-\nu_2)^2}{2\sigma_0^2}}\right) \,d\bmu =\\
    &= \int_{\mu_2} \dfrac{\mathrm{e}^{-\left({2c_1\cdot\left({\nu}_2^2-2{\mu}_2{\nu}_2+{\nu}_1^2-2{\mu}_2{\nu}_1+2{\mu}_2^2\right){\sigma}_0^2-\left(c_2-1\right)\left(c_2+1\right){\nu}_2^2+2\left(c_2-1\right)\left(c_2{\nu}_1+{\mu}_2\right){\nu}_2-\left(c_2-2\right)c_2{\nu}_1^2-2c_2{\mu}_2{\nu}_1+{\mu}_2^2}\right)/\left({2{\sigma}_0^2\cdot\left(2c_1{\sigma}_0^2+1\right)}\right)}}{\sqrt{2}\sqrt{{\pi}}\,{\sigma}_0\sqrt{2c_1{\sigma}_0^2+1}} \, d\mu_2\\
    &=  \dfrac{1}{\sqrt{4c_1{\sigma}_0^2+1}} \exp\left({-\frac{\left({\nu}_2-{\nu}_1\right)^2\left(c_1{\sigma}_0^2-c_2^2+c_2\right)}{{\sigma}_0^2\cdot\left(4c_1{\sigma}_0^2+1\right)}}\right)
\end{align*}
\end{proof}

\begin{lemma}\label{lem:gaussuan_integration2}
Let $\mu_i\sim \cN(\nu_i,\s_{0}^2), \mu_j\sim \cN(\nu_j,\s_{0}^2)$ and $C\geq 0$. We have that:
\begin{align*}
    \int_{\bmu} 
    \exp\left(-C (\mu_i-\mu_j)^2 \right)
    h(\bmu) \,d\bmu 
    = \sqrt{
    \frac{1}{4C\s_0^2+1}}
    \exp\left(-\frac{C(\nu_i-\nu_j)^2}{4\s_0^2C+1}\right)
\end{align*}
\end{lemma}
\begin{proof}
We will use the following.
Let $a,b,c\geq 0$ then:
\begin{align}\label{eq:gauss_integr}
    \int_{x=-\infty}^{+\infty} \exp\left({-ax^2+bx+c}\right) \,dx = \sqrt{\frac{\pi}{a}}\exp\left({\frac{b^2}{4a}+c}\right).
\end{align}
\noindent
Now, we can compute the objective as follows:
\begin{align*}
    &\int_{\bmu} 
    \exp\left(-C (\mu_i-\mu_j)^2 \right)
    h(\bmu) \,d\bmu =\\
    &= \frac{1}{2\pi \s_0^2} \int_{\mu_i,\mu_j} \exp\left(-C (\mu_i-\mu_j)^2 \right) \exp\left(- \frac{(\nu_i-\mu_i)^2}{2\s_0^2} \right) \exp\left(- \frac{(\nu_j-\mu_j)^2}{2\s_0^2} \right) \,d\mu_i \,d\mu_j \\
    &= \frac{1}{2\pi \s_0^2} \int_{\mu_i,\mu_j} \exp\left(-C (\mu_i-\mu_j)^2 - \frac{(\nu_i-\mu_i)^2}{2\s_0^2} - \frac{(\nu_j-\mu_j)^2}{2\s_0^2} \right) \,d\mu_i \,d\mu_j \\
    &= \frac{1}{2\pi \s_0^2} \int_{\mu_i,\mu_j} \exp\left( -\mu_i^2\left(C+\frac{1}{2\s_0^2}\right)+ \mu_i 2\left(C\mu_j+\frac{\nu_i}{2\s_0^2}\right)+\left(-C\mu_j^2+\frac{2\nu_j\mu_j-\nu_i^2-\mu_j^2-\nu_j^2}{2\s_0^2}\right)\right) \,d\mu_i \,d\mu_j \\
    &= \frac{1}{2\pi \s_0^2} \int_{\mu_j} \sqrt{\frac{\pi}{C+\frac{1}{2\s_0^2}}}\exp\left(\frac{\left(C\mu_j+\frac{\nu_i}{2\s_0^2}\right)^2}{\left(C+\frac{1}{2\s_0^2}\right)}+
    \left(-C\mu_j^2+\frac{2\nu_j\mu_j-\nu_i^2-\mu_j^2-\nu_j^2}{2\s_0^2}\right)\right) \,d\mu_j ~~~~~\text{ using \cref{eq:gauss_integr}}\\
    &= \sqrt{\frac{1}{4\pi\s_0^4C+2\pi\s_0^2}} \int_{\mu_j} 
    \exp\left(
    - 
    \frac{\mu_j^2}{2\s_0^2}
    \frac{2C+\frac{1}{2\s_0^2}}{C+\frac{1}{2\s_0^2}}
    + 
    \frac{\mu_j}{2\s_0^2}
    \frac{2\left(C\nu_i +(C+\frac{1}{2\s_0^2})\nu_j\right)}{C+\frac{1}{2\s_0^2}}
    +\frac{1}{2\s_0^2}\frac{\left(-\frac{\nu_j^2}{2\s_0^2}-C(\nu_i^2+\nu_j^2)\right) }{C+\frac{1}{2\s_0^2}}
    \right) \,d\mu_j \\
    &= \sqrt{
    \frac{1}{4C\s_0^2+1}}
    \exp\left(-\frac{C(\nu_i-\nu_j)^2}{4\s_0^2C+1}\right) \hspace{8.3cm}\text{ using \cref{eq:gauss_integr}}
    \end{align*}
\end{proof}

\section{Proofs of \cref{sec:lb'}}\label{app:LB'}

% \subsection{Proof of \cref{thm:lb'}}\label{sec:proof_lb'}
\thmlb*

\begin{proof}
Let $K=2$ arms with means drawn according to $\mathcal{N}(\nu_1,\s_0^2)\mathcal{N}(\nu_2,\s_0^2)$. For any mean vector $\bmu = (\mu_1,\mu_2)$ we define the symmetric vector $\bmu'=(\mu_2,\mu_1)$. 
% We define instance $\mathcal{I}(\bmu)$ where the reward distribution of the arms is $\cN(\mu_1,\s^2)\cN(\mu_2,\s^2)$. 

\begin{restatable}{proposition}{propchopr}\label{prop:chopr}
    We have that $h(\bmu') = \exp\left(-\frac{(\mu_1-\mu_2)(\nu_1-\nu_2)}{\sigma_0^2}\right) h(\bmu)$.
\end{restatable}

% Let $d_{KL,1}(\cI(\bmu),\cI(\bmu'))$ be the KL-divergence of  
% the distribution of arm $1$ in instance $\mathcal{I}(\bmu)$ compared to the distribution of the same arm in $\mathcal{I}(\bmu')$. Since the distributions are Gaussian with identical variances $\s^2$, we have that
% \begin{align*}
%     d_{KL,1}(\cI(\bmu),\cI(\bmu')) = \frac{(\mu_1-\mu_2)^2}{2\s^2}. 
% \end{align*}

Let $\{X_{1,s}\}_{s\in[t]}$ be $t$ samples from arm $1$ and $\{X_{2,s}\}_{s\in[t]}$ be $t$ samples from arm $2$. We define the following:
\begin{align*}
    \hat d_{KL}^{(1,t)} = \frac{1}{t} \sum_{s=1}^t \frac{(\mu_1-\mu_2)(2X_{1,s} - \mu_1 - \mu_2)}{2\s^2}\, \text{ and }\, \hat d_{KL}^{(2,t)} = \frac{1}{t} \sum_{s=1}^t \frac{(\mu_2-\mu_1)(2X_{2,s} - \mu_2 - \mu_1)}{2\s^2}. 
\end{align*}
These are essentially unbiased estimators of the KL-divergence between instances $\bmu, \bmu'$.
% $d_{KL,1}(\cI(\bmu),\cI(\bmu'))$.
% , between the reward distributions of arm $1$ in instances $\cI(\bmu)$ and $\cI(\bmu')$
% Similarly, for arm $2$, we have $d_{KL,2}
% (\cI(\bmu),\cI(\bmu'))=\frac{(\mu_2-\mu_1)^2}{2\s^2}$ and given $t$ samples $\{X_{2,s}\}_{s\in[t]}$, we define the estimator
% \begin{align*}
%     \hat d_{KL}^{(2,t)} = \frac{1}{t} \sum_{s=1}^t \frac{(\mu_2-\mu_1)(2X_{2,s} - \mu_2 - \mu_1)}{2\s^2}. 
% \end{align*}
The following event states that the estimators $\hat{d}_{KL}^{(1,t)},\hat{d}_{KL}^{(2,t)}$ are sufficiently close to their means at any time $t\in[n]$,
\begin{align}\label{eq:ksi}
    E =\Bigg\{\forall i \in \{1,2\}, \forall t \in [n]:  \hat d_{KL}^{(i,t)} - \frac{(\mu_1-\mu_2)^2}{2\s^2} \leq \frac{|\mu_1-\mu_2|}{\s}\sqrt{\frac{4\log(2n)}{t}}\Bigg\}.
\end{align} 
In \cref{lem:estim_concentration'}, we bound the probability of the above concentration event in instance $\cI(\bmu)$. This is a frequentist argument, which only depends on the quality of sampling from the reward distributions.
\begin{restatable}{lemma}{lemEstimConcc}(Concentration of divergence estimators)\label{lem:estim_concentration'}
We have that $\Prob{}{E^c|\bmu} \leq \frac{1}{2n}$.
\end{restatable}

In the case where the estimators of the KL divergence of the arms are concentrated according to \cref{eq:ksi}, using the change of measure identity we show the following lemma.

\begin{restatable}{lemma}{lemchomm}(Change of measure)\label{lem:chom'}
    For any event $\mathcal{E}$ defined on $n$ rounds of sampling by some policy, we have that 
    $\,
    \Prob{}{\cE\cap E|\bmu'} \geq \Prob{}{\cE\cap E|\bmu} \, \exp\left(- n(8\log(2n)+1)\frac{(\mu_1-\mu_2)^2}{2\sigma^2}-1\right).
    $
\end{restatable}

We use \cref{lem:chom'} with $\cE = \{J\neq i_*(\bmu')\}$ to obtain
\begin{align}\label{eq:lb_lasteq}
    &\Prob{}{J\neq i_*(\bmu')\cap E|\bmu'} h(\bmu')  + \Prob{}{J\neq i_*(\bmu)\cap E|\bmu} h(\bmu)  \\
    &= \left[\Prob{}{J\neq i_*(\bmu')\cap E|\bmu'} \exp\left(\frac{(\mu_1-\mu_2)(\nu_1-\nu_2)}{\sigma_0^2}\right) + \Prob{}{J\neq i_*(\bmu)\cap E|\bmu} \right]h(\bmu)  &\text{(by \cref{prop:chopr})} \nonumber\\
    &\geq \left[\Prob{}{\{J\neq i_*(\bmu')\}\cap E|\bmu} \exp\left(- n(8\log(2n)+1)\frac{(\mu_1-\mu_2)^2}{2\sigma^2}-1-\frac{(\mu_1-\mu_2)(\nu_1-\nu_2)}{\sigma_0^2}\right)\right. \nonumber\\
    &\quad + \left.\Prob{}{J\neq i_*(\bmu)\cap E|\bmu} \vphantom{\frac{(\mu_1-\mu_2)(\nu_1-\nu_2)}{2\sigma_0^2}}\right]h(\bmu)  &\text{(by \cref{lem:chom'})} \nonumber\\
    &= \left[\Prob{}{\{J= i_*(\bmu)\}\cap E|\bmu} \exp\left(- n(8\log(2n)+1)\frac{(\mu_1-\mu_2)^2}{2\sigma^2}-1-\frac{(\mu_1-\mu_2)(\nu_1-\nu_2)}{\sigma_0^2}\right)\right. \nonumber\\
    &\quad + \left.\Prob{}{J\neq i_*(\bmu)\cap E|\bmu} \vphantom{\frac{(\mu_1-\mu_2)(\nu_1-\nu_2)}{\sigma_0^2}}\right]h(\bmu) \nonumber\\
    &\geq  \Prob{}{E|\bmu}\,\exp\left(- n(8\log(2n)+1)\frac{(\mu_1-\mu_2)^2}{2\sigma^2}-1-\frac{(\mu_1-\mu_2)(\nu_1-\nu_2)}{\sigma_0^2}\right) h(\bmu) \nonumber \\
    &\geq  \left(1-\frac{1}{2n}\right)\exp\left(- n(8\log(2n)+1)\frac{(\mu_1-\mu_2)^2}{2\sigma^2}-1\right) h(\bmu')  &\text{(by \ref{lem:estim_concentration'} and \ref{prop:chopr})}\nonumber\,.
\end{align} 
Then the expected probability of misidentification can be bounded as follows
\begin{align*}
    &\int_{\bmu}\Prob{}{J\neq i_*(\bmu)|\bmu} h(\bmu) \, d\bmu\\
    &\geq \int_{\mu_1\geq\mu_2}\Prob{}{J\neq i_*(\bmu) \cap E|\bmu} h(\bmu) + \Prob{}{J\neq i_*(\bmu') \cap E|\bmu'} h(\bmu') \, d\mu_1 \,d\mu_2\\
    &\geq  \frac{1}{4e \pi\s_0^2}\int_{\mu_1\geq\mu_2}\exp\left(- n(8\log(2n)+1)\frac{(\mu_1-\mu_2)^2}{2\sigma^2}-\frac{(\mu_1-\nu_2)^2}{2\s_0^2}-\frac{(\mu_2-\nu_1)^2}{2\s_0^2}\right) \, d\mu_1 \,d\mu_2\\
    &= \frac{1}{4e \pi\s_0^2}\int_{\delta\geq 0}\exp\left(- n(8\log(2n)+1)\frac{\delta^2}{2\sigma^2}\right) \left[\int_{\mu_2=-\infty}^{\infty} \exp\left(-\frac{(\mu_2-(\nu_2-\delta))^2}{2\s_0^2}-\frac{(\mu_2-\nu_1)^2}{2\s_0^2}\right) \, d\mu_2 \right] \,d \delta \\
&= \frac{\sqrt{2}}{4e \sqrt{\pi}\s_0}\int_{\delta\geq 0}\exp\left(- n(8\log(2n)+1)\frac{\delta^2}{2\sigma^2}\right) \exp\left(-\frac{(\nu_2-\delta-\nu_1)^2}{4\s_0^2}\right)  \,d \delta\,.
\end{align*}
where in the second inequality we used \eqref{eq:lb_lasteq} with $n\geq 1$. We then defined $\delta=\mu_1-\mu_2$ and computed the Gaussian integral w.r.t. $\mu_2$. 
We upper bound the remaining integral using the following proposition, which relies on a Gaussian tail probability lower bound.
\begin{restatable}{proposition}{lbtailprob}\label{eq:lb tail prob}
When $\nu_1>\nu_2$ we have 
    \begin{align*}
        &\int_{\delta\geq 0}\exp\left(- n(8\log(2n)+1)\frac{\delta^2}{2\s^2}-\frac{(\nu_2-\delta-\nu_1)^2}{4\s_0^2}\right)  \,d \delta \\
        & \geq \sqrt{\frac{2\pi\s_0^2\s^2}{2n(8\log(2n)+1)\s_0^2+\s^2}} \exp\left(-\frac{(\nu_1-\nu_2)^2}{4\s_0^2}\right)\,.
    \end{align*}
\end{restatable}
This concludes the proof.
\end{proof}

We prove all auxiliary claims next.

\propchopr*
\begin{proof}
By definition of the Gaussian prior
    \begin{align*}
        h(\bmu') &= \frac{1}{2\pi\sigma_0^2}\exp\left(-\frac{(\mu_1-\nu_2)^2 + (\mu_2-\nu_1)^2}{2\sigma_0^2}\right)\\
        &= \frac{1}{2\pi\sigma_0^2}\exp\left(-\frac{\mu_1^2+\nu_2^2+\mu_2^2+\nu_1^2- 2(\mu_1\nu_2+\mu_2\nu_1)}{2\sigma_0^2}\right)\\
        &= \frac{1}{2\pi\sigma_0^2}\exp\left(-\frac{\mu_1^2+\nu_2^2+\mu_2^2+\nu_1^2 - 2(\mu_1\nu_1+\mu_2\nu_2) + 2(\mu_1\nu_1+\mu_2\nu_2)  - 2(\mu_1\nu_2+\mu_2\nu_1) }{2\sigma_0^2}\right)\\
        &= h(\bmu)\exp\left(-\frac{ 2(\mu_1\nu_1+\mu_2\nu_2) - 2(\mu_1\nu_2+\mu_2\nu_1) }{2\sigma_0^2}\right)
    \end{align*}
\end{proof}

\lemEstimConcc*
\begin{proof}
%     We can verify that $n_1 \hat d_{KL}^{(1,n_1)}+ n_2\hat d_{KL}^{(2,n_2)}$ is an unbiased estimate of $d_{KL}(\cI(\bmu),\cI(\bmu'))$:
% \begin{align*}
%     &\Ex{}{n_1\hat d_{KL}^{(1,n_1)}+n_2\hat d_{KL}^{(2,n_2)}|\bmu} \\
%     &= \Ex{}{n_1|\bmu}\frac{(\mu_1-\mu_2)(2\mu_1 - \mu_1 - \mu_2)}{2\sigma^2} \\
%     &\quad+ \Ex{}{n_2|\bmu}\frac{(\mu_2-\mu_1)(2\mu_2 - \mu_2 - \mu_1)}{2\sigma^2} \\
%     &= n\frac{(\mu_1-\mu_2)^2}{2\sigma^2}.
% \end{align*}
We have that:
\begin{align*}
    &E =\Bigg\{\forall i \in \{1,2\}, \forall t \in [n]:  |\hat d_{KL}^{(i,t)}| - \frac{(\mu_1-\mu_2)^2}{2\sigma^2} \leq \frac{|\mu_1-\mu_2|}{\s}\sqrt{\frac{4\log(2n)}{t}}\Bigg\}.
\end{align*}
Recall that
\begin{align*}
    \hat d_{KL}^{(1,t)} &= \frac{1}{t} \sum_{s=1}^t \log\left(\frac{e^{-\frac{(X_{1,s}-\mu_1)^2}{2\sigma^2}}}{e^{-\frac{(X_{1,s}-\mu_2)^2}{2\sigma^2}}}\right) = \frac{1}{t} \sum_{s=1}^t \frac{(\mu_1-\mu_2)(2X_{1,s} - \mu_1 - \mu_2)}{2\sigma^2}
\end{align*}
and 
\begin{align*}
    \hat d_{KL}^{(2,t)} = \frac{1}{t} \sum_{s=1}^t \frac{(\mu_2-\mu_1)(2X_{2,s} - \mu_2 - \mu_1)}{2\sigma^2}. 
\end{align*}
Taking conditional expectation on the divergence estimators, due to linearity we get:
\begin{align*}
    \Ex{}{\hat d_{KL}^{(1,t)}|\bmu} &= \Ex{}{\frac{1}{t} \sum_{s=1}^t \frac{(\mu_1-\mu_2)(2X_{1,s} - \mu_1 - \mu_2)}{2\sigma^2}\bigg|\bmu}\\
    &= \frac{1}{t} \sum_{s=1}^t \frac{(\mu_1-\mu_2)(2\mu_{1} - \mu_1 - \mu_2)}{2\sigma^2}=\frac{(\mu_1-\mu_2)^2}{2\sigma^2}
\end{align*}
and similarly,
\begin{align*}
    \Ex{}{\hat d_{KL}^{(2,t)}|\bmu} = \Ex{}{\frac{1}{t} \sum_{s=1}^t \frac{(\mu_2-\mu_1)(2X_{2,s} - \mu_2 - \mu_1)}{2\sigma^2}\bigg|\bmu} = \frac{(\mu_1-\mu_2)^2}{2\sigma^2}.
\end{align*}
Moreover, since $X_{i,s}$ for any $i\in\{1,2\}$ is a Gaussian random variable with variance $\s^2$, the expression $ \frac{(\mu_1-\mu_2)(2X_{1,s} - \mu_1 - \mu_2)}{2\sigma^2}$ is Gaussian with standard deviation $\frac{|\mu_1-\mu_2|}{\s}$. Therefore, by using Hoeffding's inequality:
\begin{align*}
    \Prob{}{\frac{1}{t} \sum_{s=1}^t \frac{(\mu_1-\mu_2)(2X_{1,s} - \mu_1 - \mu_2)}{2\sigma^2} - \frac{(\mu_1-\mu_2)^2}{2\sigma^2} \geq \epsilon |\bmu} \leq \exp\left(-\frac{t\epsilon^2}{2 \frac{|\mu_1-\mu_2|^2}{\s^2}}\right). 
\end{align*}
Using $\epsilon= \frac{|\mu_1-\mu_2|}{\s}\sqrt{\frac{4\log(2n)}{t}}$ for any specific $t\in[n]$ we obtain that:
\begin{align*}
    \Prob{}{\frac{1}{t} \sum_{s=1}^t \frac{(\mu_1-\mu_2)(2X_{1,s} - \mu_1 - \mu_2)}{2\sigma^2} - \frac{(\mu_1-\mu_2)^2}{2\sigma^2} \geq \frac{|\mu_1-\mu_2|}{\s}\sqrt{\frac{4\log(2n)}{t}}\bigg|\bmu} \leq \frac{1}{4n^2}.
\end{align*}
Similarly for the case of the other arm $i=2$. Then, the lemma follows by taking union bound for all $t\in[n]$ for $i\in\{1,2\}$.
\end{proof}

\lemchomm*
\begin{proof}
Let $d_{KL,1}(\cI(\bmu),\cI(\bmu'))$ be the KL-divergence of  
the distribution of arm $1$ in instance $\mathcal{I}(\bmu)$ compared to the distribution of the same arm in $\mathcal{I}(\bmu')$. Since the distributions are Gaussian with identical variances $\s^2$, we have that
\begin{align*}
    d_{KL,1}(\cI(\bmu),\cI(\bmu')) = \frac{(\mu_1-\mu_2)^2}{2\s^2}. 
\end{align*}

Let $\{X_{1,s}\}_{s\in[t]}$ be $t$ samples from arm $1$. We have defined the following:
\begin{align*}
    \hat d_{KL}^{(1,t)} = \frac{1}{t} \sum_{s=1}^t \log\left(\frac{e^{-\frac{(X_{1,s}-\mu_1)^2}{2\s^2}}}{e^{-\frac{(X_{1,s}-\mu_2)^2}{2\s^2}}}\right) = \frac{1}{t} \sum_{s=1}^t \frac{(\mu_1-\mu_2)(2X_{1,s} - \mu_1 - \mu_2)}{2\s^2}.
\end{align*}
which is an unbiased estimator for the KL-divergence $d_{KL,1}(\cI(\bmu),\cI(\bmu'))$.
% , between the reward distributions of arm $1$ in instances $\cI(\bmu)$ and $\cI(\bmu')$

Similarly, for arm $2$, we have $d_{KL,2}
(\cI(\bmu),\cI(\bmu'))=\frac{(\mu_2-\mu_1)^2}{2\s^2}$ and given $t$ samples $\{X_{2,s}\}_{s\in[t]}$, we define the estimator
\begin{align*}
    \hat d_{KL}^{(2,t)} = \frac{1}{t} \sum_{s=1}^t \frac{(\mu_2-\mu_1)(2X_{2,s} - \mu_2 - \mu_1)}{2\s^2}. 
\end{align*}

Then, the following event states that the estimators $\hat{d}_{KL}^{(1,t)},\hat{d}_{KL}^{(2,t)}$ are sufficiently close to their means at any time $t\in[n]$,
\begin{align*}
    E =\Bigg\{\forall i \in \{1,2\}, \forall t \in [n]:  \hat d_{KL}^{(i,t)} - \frac{(\mu_1-\mu_2)^2}{2\s^2} \leq \frac{|\mu_1-\mu_2|}{\s}\sqrt{\frac{4\log(2n)}{t}}\Bigg\}.
\end{align*} 
We bound the probability of the above concentration event in instance $\cI(\bmu)$. This is a frequentist argument, which only depends on the quality of sampling from the reward distributions.

We consider a policy $\pi$ that collects samples from the two arms for $n$ rounds. Suppose that arms $1$ and $2$ are sampled $n_1$ and $n_2$ times, respectively, such that $n_1+n_2=n$.
% Since the sampling from the reward distributions is independent in each round, we have that the KL divergence between the $n$-step sampling from $\cI(\bmu)$ compared to $\cI(\bmu')$ is:
% \begin{align*}
%     &\Ex{}{n_1|\bmu} \frac{(\mu_1-\mu_2)^2}{2\s^2} + \Ex{}{n_2|\bmu} \frac{(\mu_2-\mu_1)^2}{2\s^2}\\
%     &= n\frac{(\mu_1-\mu_2)^2}{2\s^2}.
% \end{align*}
By the change of measure identity, for any event defined on $n$ rounds of sampling by $\pi$ from the two arms, we have that:
\begin{align*}
    &\Prob{}{\cE\cap E|\bmu} \\
    &= \Ex{}{\mathbbm{1}(\cE\cap E) \cdot e^{-n_1 \hat d_{KL}^{(1,n_1)}-n_2\hat d_{KL}^{(2,n_2)}}\big|\bmu}\\
    &\geq \Ex{}{\mathbbm{1}(\cE\cap E) \cdot e^{-n_1 \frac{(\mu_1-\mu_2)^2}{2\sigma^2} - n_1\frac{|\mu_1-\mu_2|}{\s}\sqrt{\frac{4\log(2n)}{n_1}}-n_2 \frac{(\mu_1-\mu_2)^2}{2\sigma^2} - n_2\frac{|\mu_1-\mu_2|}{\s}\sqrt{\frac{4\log(2n)}{n_2}}}\bigg|\bmu} ~~~~ \text{(by definition of $E$)}\\
    &\geq \Ex{}{\mathbbm{1}(\cE\cap E) \cdot e^{- n\frac{(\mu_1-\mu_2)^2}{2\sigma^2} - \sqrt{4n\log(2n) \frac{(\mu_1-\mu_2)^2}{\s^2}}}\bigg|\bmu} ~~~~~ \text{(due to $n_1+n_2=n$ and $n_1,n_2\geq 0$)}\\
    &= \Prob{}{\cE\cap E|\bmu} e^{- n\frac{(\mu_1-\mu_2)^2}{2\sigma^2} - \sqrt{4n\log(2n) \frac{(\mu_1-\mu_2)^2}{\s^2}}}\,.
\end{align*}
Moreover, we use the following bound
\begin{align}\label{eq:lb_x_iden}
    &\exp\left({- n\frac{(\mu_1-\mu_2)^2}{2\sigma^2}-\sqrt{4n\log(2n) \frac{(\mu_1-\mu_2)^2}{\s^2}}}\right)\nonumber\\
    &\geq \exp\left({- n\frac{(\mu_1-\mu_2)^2}{2\sigma^2}-4n\log(2n) \frac{(\mu_1-\mu_2)^2}{\s^2}-1}\right) \qquad\text{(due to $x<x^2+1$)}\nonumber\\
     &\geq  \exp\left({- n(8\log(2n)+1)\frac{(\mu_1-\mu_2)^2}{2\sigma^2}-1}\right).
\end{align} 
Thus, the lemma follows.

\end{proof}

\lbtailprob* 
\begin{proof}
We have that 
\begin{align*}
    &\int_{\delta\geq 0}\exp\left(- n(8\log(2n)+1)\frac{\delta^2}{2\sigma^2}-\frac{(\nu_2-\delta-\nu_1)^2}{4\s_0^2}\right)  \,d \delta\\
    &= \int_{\delta\geq 0}\exp\left(- \frac{2n(8\log(2n)+1)\s_0^2+\s^2}{2\sigma^2 \s_0^2} \left(\delta - \frac{\sigma^2 (\nu_2-\nu_1)}{4n(8\log(2n)+1)\s_0^2+2\s^2}  \right)^2\right)  \,d \delta\\
    &\times \exp\left(-\frac{2n(8\log(2n)+1)\s_0^2}{2n(8\log(2n)+1)\s_0^2+\s^2}\frac{(\nu_1-\nu_2)^2}{4\s_0^2}\right)
\end{align*}
Now we use that for a Gaussian random variable $X$ with zero mean and variance $\s'$ and $x>0$ we have 
\begin{align*}
    \int_{X\geq x} e^{-\frac{(X-x)^2}{2\s'^2}} \geq {\sqrt{2\pi}\s'} \exp\left(-(x/\s')^2\right)
\end{align*}
Then for $x=\frac{\sigma^2}{4n(8\log(2n)+1)\s_0^2+2\s^2}(\nu_2-\nu_1),\s'=\sqrt{\frac{\sigma^2 \s_0^2}{2n(8\log(2n)+1)\s_0^2+\s^2}}$ we get
\begin{align*}
    &\int_{\delta\geq 0}\exp\left(- \frac{2n(8\log(2n)+1)\s_0^2+\s^2}{2\sigma^2 \s_0^2} \left(\delta - \frac{\s^2 (\nu_2-\nu_1)}{4n(8\log(2n)+1)\s_0^2+2\s^2}  \right)^2\right)  \,d \delta\\
    &\geq \sqrt{2\pi}\sqrt{\frac{\sigma^2 \s_0^2}{2n(8\log(2n)+1)\s_0^2+\s^2}} \exp\left(-\frac{\s^2}{2n(8\log(2n)+1)\s_0^2+\s^2}\frac{(\nu_1-\nu_2)^2}{4\s_0^2}\right)
\end{align*}    
\end{proof}

\thmFreqLBound*

\begin{proof}

Let the mean vector $\bm{\mu}=(\mu_1,...,\mu_K)$, where $X_i\sim\cN(\mu_i,1)$ and $\mu_i\sim\mathcal{N}(\nu_i,\s_0^2)$.
We would like to bound the probability that the learner fails to recommend the optimal arm when presented with instance $\bmu$, i.e. $\Ex{H}{\Prob{}{J\neq i_*~|~\bmu}}$. We also assume that the learner is oblivious to prior information. 

Instead, we consider the easier problem where the learner is required to distinguish the best arm given that the instance it faces is one of $K$ instances defined as follows:
We define $d_i = \mu_{i_*}-\mu_{i}$. 
Let $K$ pairs of corresponding Gaussian distributions $p_i= \mathcal{N}(\mu_i,1)$ and $p_i'=  \mathcal{N}( \mu_i',1)$ where $\mu_i'=\mu_i + 2d_i$. 
We define $K$ Gaussian bandit instances where, in  instance $i\in [K]$, any arm $k\in [K]$ has distribution $p_{k}^i=p_{k}$ if $i\neq k$ and $p_{k}^i=p_{k}'$ if $k=i$, i.e. arm $i$ is uniquely the best arm in instance $i$. 

We define the product distribution of instance $i \in [K]$ as $G^i=p_{1}^i\times \dots \times p_{K}^i$. 

For $i\in[K]$, we use the notation $\Prob{i}{\cdot} = \Prob{G^i}{\cdot|\bmu}$ and $\Ex{i}{\cdot} = \Ex{G^i}{\cdot | \bmu}$ to denote the probability (resp. expectation) w.r.t. the randomness of sampling in instance $i$. \\
Let the KL divergence between distributions $p,p'$: 
\begin{align*}
    d_{KL}(p,p') = \int_{-\infty}^{\infty} \log \left(\frac{\,d p(x)}{\,d p'(x)}\right) \,d p(x)
\end{align*} 
Here, for $k\in[K]$ we define the KL divergence for arm $k$ as: 
$$d_{KL}^{k} = d_{KL}(p_k,p_k') = d_{KL}(p_k',p_k) = \frac{(\mu_k-\mu_k')^2}{2} = \frac{(2d_k)^2}{2} = 2d_k^2.$$
For $t\in [n], k \in [K]$ let $t$ samples $\{X_{k,s}\}_{s\in[n]}\sim p_k^i$ from arm $k$ in some bandit instance $i$. Moreover, let the empirical KL divergence computed from the samples of arm $k$:
\begin{align*}
     \widehat{d}_{KL}^{k,t} &= \frac{1}{t} \sum_{s\in[t]} \log \left(\frac{\,d p_k}{\,d p_k'}(X_{k,s})\right) = \frac{1}{t} \sum_{s\in[t]} \log \left(\frac{\frac{1}{\sqrt{2\pi}} e^{\frac{-(X_{k,s}-\mu_k)^2}{2}}}{\frac{1}{\sqrt{2\pi}} e^{\frac{-(X_{k,s}-\mu'_k)^2}{2}}}\right) \\
     &= \frac{1}{t} \sum_{s\in[t]} \left( -\frac{(X_{k,s}-\mu_k)^2}{2} + \frac{(X_{k,s}-\mu_k')^2}{2} \right) \\
     &= \frac{1}{t} \sum_{s\in[t]} \frac{(2 X_{k,s} -(\mu_k+\mu_k')) (\mu_k'-\mu_k)}{2}  \\
     &= \frac{1}{t} \sum_{s\in[t]} 2(X_{k,s} -\mu_{i_*}) d_k
\end{align*}
Note that $\Ex{G^i}{\widehat{d}_{KL}^{k,t}} = 2(\mu_k-\mu_{i_*})d_k=-d_{KL}^k$ if $k\neq i$, or $\Ex{G^i}{\widehat{d}_{KL}^{k,t}} = 2(\mu_k'-\mu_{i_*})d_k=2(\mu_k+2d_k-\mu_{i_*})d_k=d_{KL}^k$ if $k=i$. Therefore, $|\widehat{d}_{KL}^{k,t}|$ is an unbiased estimator of $d_{KL}^k$. Moreover, we have the following concentration result:
\begin{lemma}\label{lem:KL_concentration}
Let 
\begin{align*}
    \Xi = \left\{|\widehat{d}_{KL}^{k,t}| - d_{KL}^k \leq 2d_k\sqrt{\frac{2\log(6Kn)}{t}}, \forall k\in [K], t\in [n]\right\}
\end{align*}
For any $i\in[K]$ we have that:
 \begin{align*}
     \Prob{i}{\Xi} \geq 5/6
 \end{align*}
\end{lemma}
\begin{proof}
Since $X_{k,s}\sim p_k^i$, the quantity $2(X_{k,s} -\mu_{i_*}) d_k$ is a Gaussian random variable. In particular, if $k=i$ then $2(X_{k,s} -\mu_{i_*}) d_k\sim \mathcal{N}\left(2d_k^2, 2d_k\right)=\mathcal{N}\left(d_{KL}^k, 2d_k\right)$. 
On the other hand, if $k\neq i$ 
then $2(X_{k,s} -\mu_{i_*}) d_k\sim \mathcal{N}\left(-2d_k^2, 2d_k\right)=\mathcal{N}\left(-d_{KL}^k, 2d_k\right)$.
% Moreover, $2d_k
% \leq \frac{1}{2}$ since $\mu_k\in[1/4,1/2)$.
Thus using Hoeffding's inequality for the empirical mean of subgaussian random variables:
\begin{align*}
    % \Prob{G^i}{|\widehat{KL}_{k,t}| - KL_k \geq \frac{1}{2}\sqrt{\frac{2\log (1/\delta)}{t}}} 
    % \leq 
    \Prob{G^i}{|\widehat{d}_{KL}^{k,t}| - d_{KL}^k \geq 2d_k \sqrt{\frac{2\log (1/\delta)}{t}}} \leq \delta
\end{align*}
Using $\delta=(6nK)^{-1}$ and union bound over all $t\in [n]$ and $k\in [K]$ we obtain that:
\begin{align*}
    \Prob{G^i}{|\widehat{d}_{KL}^{k,t}| - d_{KL}^k \leq 2d_k\sqrt{\frac{2\log(6Kn)}{t}}, \forall k\in [K], t\in [n]} \geq 5/6
\end{align*}
\end{proof}
We consider an algorithm that returns arm $J$ and denote by $n_i$ the number of times arm $i$ has been pulled. As in \cite{Carpentier16}, we define:
\begin{equation}\label{eq:def_times}
    t_i = \Ex{i_*}{n_i}
\end{equation} 
and the event:
\begin{align*}
    \mathcal{E}_i = \{J= i_*\}\cap \Xi \cap\{n_i\leq 6 t_i\}.
\end{align*} 
We focus on some $i\in[K]$. Using the change of measure identity, since the distributions $G_{i_*},G_i$ only differ in arm $i$, we have that:
\begin{align*}
    \Prob{i}{\mathcal{E}_i} = \Ex{i_*}{\bm{1}\{\mathcal{E}_i\} \exp\left(-n_i \widehat{d}_{KL}^{i,n_i}\right)}
\end{align*} 
Using \cref{lem:KL_concentration} and subsequently \cref{eq:def_times} we get that:
\begin{align*}
    \Prob{i}{\mathcal{E}_i} 
    &= \Ex{i_*}{\bm{1}\{\mathcal{E}_i\} \exp\left(-n_i \widehat{d}_{KL}^{i,n_i}\right)}\\
    &\geq \Ex{i_*}{\bm{1}\{\mathcal{E}_i\} \exp\left(-n_id_{KL}^{i} - 2d_i\sqrt{2n_i\log(6Kn)}\right)} \\
    &\geq \Ex{i_*}{\bm{1}\{\mathcal{E}_i\} \exp\left(-6t_id_{KL}^{i} - 2d_i\sqrt{12t_i\log(6Kn)}\right)} \\
    &= \Prob{i_*}{\mathcal{E}_i} \exp\left(-6t_id_{KL}^{i} - 2\sqrt{12t_id_i^2\log(6Kn)}\right) \\
    &= \Prob{i_*}{\mathcal{E}_i} \exp\left(-12 t_i d^2_i - 2\sqrt{12 t_i d_i^2\log(6Kn)}\right) \\
\end{align*}
Recall that we are interested in bounding the probability of error, i.e. the following quantity, for any $i\in[K]$:
\begin{align*}
    \Ex{H}{\Prob{i}{J\neq i}}. 
\end{align*}
Notice that for $i\in [K]\setminus \{i_*\}$ the probability of error in instance $i$ can be lower bounded as follows:
\begin{align*}
    \Prob{i}{J\neq i}
    \geq \Prob{i}{\mathcal{E}_i} 
    \geq \Prob{i_*}{\mathcal{E}_i} \exp\left(-12 t_i d^2_i - 2\sqrt{12 t_i d_i^2\log(6Kn)}\right).
\end{align*}
Since $\sum_{i\in[K]} t_i = n$ and let $D=\sum_{i\in [K]\setminus \{i_*\}} \frac{1}{d_i^2}$ then $\exists i\in[K]\setminus\{i_*\}$ such that $t_i d_i^2 \leq n/D$. Thus, there exists $i\in [K]\setminus\{i_*\}$ such that:
\begin{align*}
    \Prob{i}{J\neq i} \geq \Prob{i_*}{\mathcal{E}_i} \exp\left(-12\frac{n}{D} - 2\sqrt{12\frac{n}{D}\log(6Kn)}\right).
\end{align*}
Using that, by Markov's inequality $\Prob{i_*}{n_i\geq 6 t_i}\leq \frac{1}{6}$ and \cref{lem:KL_concentration} we get that $$\Prob{i_*}{\mathcal{E}_i}\geq 1 - \Prob{i_*}{n_i\geq 6 t_i} - \Prob{i_*}{\Xi} - \Prob{i_*}{J\neq i_*} \geq 1-\frac{1}{6}-\frac{1}{6} - \Prob{i_*}{J\neq i_*}= \frac{2}{3}-\Prob{i_*}{J\neq i_*}.$$ 
Thus, for this $i\in[K]\setminus {i_*}$ we have that
\begin{align*}
    \Prob{i}{J\neq i} 
    &\geq \left(\frac{2}{3}-\Prob{i_*}{J\neq i_*}\right) \exp\left(-12\frac{n}{D} - 2\sqrt{12\frac{n}{D}\log(6Kn)}\right) \\
    &\geq \frac{2}{3} \exp\left(-12\frac{n}{D} - 2\sqrt{12\frac{n}{D}\log(6Kn)}\right) -\Prob{i_*}{J\neq i_*}.
\end{align*}
Rearranging the terms in the above inequality we get that
\begin{align*}
    \Prob{i}{J\neq i} + \Prob{i_*}{J\neq i_*} \geq \frac{2}{3} \exp\left(-12\frac{n}{D} - 2\sqrt{12\frac{n}{D}\log(6Kn)}\right).
\end{align*}
Then, $\exists i\in[K]$ such that:
\begin{align*}
    \Prob{i}{J\neq i} \geq \frac{1}{3} \exp\left(-12\frac{n}{D} - 2\sqrt{12\frac{n}{D}\log(6Kn)}\right).
\end{align*}
Then, selecting prior means such that $\nu_i>\nu_j\forall j\in[K]\setminus i$, and  
taking expectation over the prior $\mu_i\sim \mathcal{N}(\nu_i,\s_0^2)$ we get that:
\begin{align*}
    \Ex{H}{\Prob{i}{J\neq i}} \geq \frac{1}{3} \Ex{H}{\exp\left(-12\frac{n}{D} - 2\sqrt{12\frac{n}{D}\log(6Kn)}\right)}.
\end{align*} 

The theorem follows by taking $\s_0\rightarrow 0$.
\end{proof}

\end{document}